\documentclass[twoside]{article}

\usepackage[preprint]{aistats2026}
\usepackage{amsmath}
\usepackage{amssymb}
\usepackage{mathtools}
\usepackage{amsthm}
\usepackage{graphicx}
\usepackage{array}
\usepackage{makecell}
\usepackage{xcolor}
\usepackage{natbib}
\usepackage{subcaption}
\usepackage{stackengine}
\usepackage{placeins}
\usepackage{booktabs}
\usepackage{enumitem}
\usepackage{hyperref}
\usepackage{siunitx}
\usepackage{listings}
\usepackage{xcolor}
\lstdefinestyle{py}{
  language=Python,
  basicstyle=\ttfamily\small,
  numbers=left,
  numberstyle=\tiny,
  stepnumber=1,
  showstringspaces=false,
  breaklines=true,
  keywordstyle=\color{blue!70!black}\bfseries,
  commentstyle=\color{gray!70}\itshape,
  stringstyle=\color{teal!60!black}
}

\DeclareMathOperator*{\argmax}{\arg\max}

\usepackage[capitalize,noabbrev]{cleveref}

\theoremstyle{plain}
\newtheorem{theorem}{Theorem}[section]
\newtheorem{proposition}[theorem]{Proposition}

\theoremstyle{definition}
\newtheorem{definition}[theorem]{Definition}

\theoremstyle{remark}
\newtheorem{remark}[theorem]{Remark}

\begin{document}

%
\runningtitle{Matricial Free Loss Regularized Autoencoders}

%

\twocolumn[

\aistatstitle{Matricial Free Energy as a Gaussianizing Regularizer: Enhancing Autoencoders for Gaussian Code Generation}

\aistatsauthor{Rishi Sonthalia \And Raj Rao Nadakuditi}
\aistatsaddress{Department of Mathematics, Boston College \\ sonthal@bc.edu \And EECS, University of Michigan \\ rajnrao@umich.edu }]

\begin{abstract}
 We introduce a novel regularization scheme for autoencoders based on matricial free energy. Our approach defines a differentiable loss function in terms of the singular values of the code matrix (code dimension × batch size). From the standpoint of free probability and random matrix theory, this loss achieves its minimum when the singular value distribution of the code matrix coincides with that of an appropriately scaled random matrix with i.i.d. Gaussian entries. Empirical simulations demonstrate that minimizing the negative matricial free energy through standard stochastic gradient–based training yields Gaussian-like codes that generalize across training and test sets. Building on this foundation, we propose a matricial free energy maximizing autoencoder that reliably produces Gaussian codes and showcase its application to  underdetermined inverse problems.
\end{abstract}

\section{Introduction}

Autoencoders \citep{salakhutdinov2006reducing, vincent2009extracting} are foundational in unsupervised representation learning, providing a flexible means to compress and reconstruct high-dimensional data. However, the latent code space is often unstructured, resulting in representations that are erratic or hard to interpret. Gaussian latent representations are desirable for promoting regularity and interpretability.
%
Traditional methods for enforcing Gaussianity impose stringent architectural constraints, such as diffeomorphic flows with tractable Jacobians or score estimation techniques. Notably, normalizing flows and score-based approaches require the latent dimension to match the data dimension, which is suboptimal when data resides on a low-dimensional manifold embedded in high-dimensional space.

We propose a \emph{matricial Free Loss} regularizer that shapes the singular value distribution of the \emph{batch code matrix} from an encoder. This loss is a discrete adaptation of a variational principle from free probability \citep{hiai2006semicircle, edelman2005random}, yielding a differentiable, architecture-agnostic objective. Its minimizers align with the spectral statistics of i.i.d.\ Gaussian codes, without needing invertibility, tractable Jacobians, or restricting to specific architectures, while remaining compatible with stochastic gradient descent.

\textbf{Contributions.} The main contributions are:
\begin{itemize}[leftmargin=1.25em, nosep, itemsep = 2pt]
\item \textbf{Free Loss:} A discrete matricial free-energy objective (Eq.~\ref{eq:freeloss-normal}) derived on the Mar\v{c}enko--Pastur maximization principle that is  differentiable and easily integrated into any encoder.
\item \textbf{General Applicability:} Training encoders and autoencoders with this regularizer across state-of-the-art models (e.g., Transformers, Conformers, EfficientNet) without bijectivity or Jacobian requirements, handling variable-length inputs and modalities like audio, text, and images.
\item \textbf{Generalization:} Demonstration that the regularizer produces Gaussian-like codes on train and test data, evaluated via scalar, vectorial, and matricial metrics.
\item \textbf{Inverse Problems:} Leveraging i.i.d.\ standard normal codes for a quadratic latent prior, enabling a recovery objective that outperforms Tikhonov-regularized autoencoders.
\end{itemize}

\paragraph{Prior Work.}
Gaussianization transforms data to approximate a Gaussian distribution. We review key approaches below. Classical methods alternate marginal Gaussianization with linear orthonormal transforms to decorrelate dimensions \citep{chen2000gaussianization, laparra2011iterative}. Recent analyses provide non-asymptotic convergence rates using random rotations, showing rapid approximation to spherical Gaussians \citep{draxler2023convergence}. 

\textbf{Normalizing Flows.} These techniques learn invertible maps $F: \mathbb{R}^d \to \mathbb{R}^d$ to transform data distributions into $\mathcal{N}(0,I)$, trained via the maximum likelihood inspired loss:
\[
    \mathcal{L}_{\text{flow}} = \frac{1}{n} \sum_{i=1}^n \left( \frac{1}{2} \|F(x_i)\|_2^2 - \log |\det J_F(x_i)| \right),
\]
which is equivalent to minimizing the KL divergence. Tractable architectures include NICE \citep{dinh2014nice}, RealNVP \citep{dinh2016density}, Glow \citep{kingma2018glow}, MAF \citep{papamakarios2017masked}, and spline flows \citep{durkan2019neural}. Extensions like SurVAE flows add surjective layers for dimension reduction \citep{nielsen2020survae}, while free-form flows enable likelihood training by quickly estimating the Jacobian \citep{draxler2024freefrom}.
\emph{Our approach relaxes diffeomorphism requirements, allowing Gaussianization with any network via a spectral loss.}

\textbf{Score Matching.} For $y = F(x)$ with density $q_F$, this minimizes the Fisher divergence to $\mathcal{N}(0,I)$:
\[
    \mathcal{L}_{\mathrm{SM}}(F) = \mathbb{E}_{y \sim q_F} \| \nabla_y \log q_F(y) + y \|_2^2,
\]
avoiding normalization constants \citep{hyvarinen2005estimation}. Variants include denoising score matching \citep{pascal2011connection} and sliced versions \citep{song2019ssm}. They have close connections to diffusion models \citep{song2019generative, song2020sde}.

\textbf{Other Techniques.} Include Wasserstein autoencoders \citep{tolstikhin2017wasserstein, kolouri2018sliced} and noise-contrastive estimation \citep{gutmann2010noise, gutmann2012noise}.

\textbf{Random matrix theory in ML.}
Random matrix theory has been used to analyze the performance of linear regression \citep{Dobriban2015HighDimensionalAO, Hastie2019SurprisesIH, Derezinski2020ExactEF, xiao2022precise,li2024least, kausik2024double, sonthalia2023training,pmlr-v238-wang24l} and deep networks as in \citep{liao2025random,adlam2019random,li2019random,baskerville2022appearance, pennington2017geometry, granziol2022random}. RMT has also been used to study the implicit regularization phenomenon and analysze the SGD algorithm by \citep{martin2021implicit, paquette20244p3,granziol2022learning} and to study the spectra of Hessian as in \citep{liao2021hessian} and \citep{ben2025spectral}. 

As far as we know, the matricial free energy function proposed in Section \ref{sec: new loss function} has not been used in the context of deep learning. \citet{nadakuditi2023free} utilize the free entropy function, which is the first of the two terms that  appear in our Free Loss function in (\ref{eq:freeloss-normal}), to develop the matricial analog of independent component analysis that they call free component analysis. 

\section{A new Gaussianizing loss function} \label{sec: new loss function}

In this section, we introduce our new matricial loss function. We begin by introducing some pertinent from random matrix theory. 

\begin{definition}[Empirical Spectral Distribution (ESD)] 
Let $X$ be a symmetric or Hermitian matrix $d \times d$ matrix with eigenvalues $\lambda_1 \geq \lambda_2 \geq \ldots \lambda_d$. 

Then, the Empirical Spectral distribution (ESDS) of $X$  is defined as:
\begin{equation}\label{eq:esd}
    \mu_{X} := \dfrac{1}{d} \sum_{i=1}^d \delta_{\lambda_i}
\end{equation}
where $\delta_{\lambda_i}$ is the Dirac delta measure at $\lambda_i$.   
\end{definition}

\begin{definition}[Mar\v{c}enko-Pastur distribution] 
Let $0 < c \leq 1$ be a shape parameter . Then, the Mar\v{c}enko-Pastur distribution \citep{marenko1967distribution} with shape  parameter $c$ has density given by: 
\begin{equation}\label{eq:mandp}
    \mu^{\textsf{M-P}}_{c}(\lambda) = \dfrac{\sqrt{(a_+ - \lambda)(\lambda-a_-)}}{2\pi c\lambda} \mathbf{1}_{[a_-,a_+]}(\lambda) d\lambda 
\end{equation}
where $a_\pm = (1 \pm \sqrt{c})^2$ are the left and right endpoints, respectively, of the distribution's support  and $\mathbf{1}_{[a_-,a_+]}(\lambda)$ is the indicator functions on $[a_-,a_+]$.
\end{definition}

\begin{proposition}\label{prop:mandp limit}
    Let $G \in \mathbb{R}^{d \times b}$ be a $d \times b$ Gaussian matrix with i.i.d. zero mean, unit variance entries - in other words, $G_{ij} \sim \mathcal{N}(0,1)$.  

Let $X = GG^T/b$ be the $d \times d$ sample covariance matrix from the Gaussian random matrix $G$. Then, as $d, b(d) \to \infty$ with $d/b(d) \to c \in (0,1]$, we have that 
\[
    \mu_{X} \overset{a.s.}{\longrightarrow} \mu^{\textsf{M-P}}_c,
\]
where $\mu^{\textsf{M-P}}_c$ is the Mar\v{c}enko-Pastur distribution in (\ref{eq:mandp}) and $\overset{a.s.}{\longrightarrow}$ denotes almost sure convergence.
\end{proposition}
\begin{proof}
This result was first established by 
\citet{marenko1967distribution} who proved convergence in probability.  \citet{silverstein1995empirical} established almost sure convergence and showed that the limiting distribution arises whenever $G_{ij}$ has zero mean, unit variance entries with bounded higher order moments.
\end{proof}

\textbf{Variational characterization of $\mu_{c}^{\textsf{M-P}}$.}

The Mar\v{c}enko-Pastur distribution is rescaled version of the Free Poisson distribution from free probability theory \citep{hiai2006semicircle, mingo2017free}.  It can be characterized as the solution to a variational problem as described next.

\begin{definition}[Voiculescu Free Entropy] Let $\mu$ be a probability measure on $\mathbb{R}$. The Voiculescu free entropy (\citep{voiculescu1997analogues}) is given by:
\begin{equation}\label{eq:chi_mu}
    \chi(\mu) = \int \log|\lambda-\tilde{\lambda}| d\mu(\lambda)d\mu(\tilde{\lambda}).
\end{equation}  
\end{definition}

\begin{proposition}[Maximization Principle for the Mar\v{c}enko-Pastur distribution]\label{prop:free entropy max}

Let $\mu$ be a probability measure on $\mathbb{R}$ and $c \in (0,1]$.  Consider the free entropy functional $\Phi_c(\mu)$ defined as:
\begin{equation}\label{eq:free entropy}
    \Phi_c(\mu) = \chi(\mu) - \int \left(\frac{\lambda}{c} - \left(\frac{1}{c}-1\right)\log(\lambda)\right)d\mu(\lambda)
\end{equation}
where $\chi(\mu)$ is the Voiculescu free entropy in (\ref{eq:chi_mu}). Then,
\begin{equation}
\mu^{\textsf{M-P}}_c  = \argmax_{\mu} \Phi_c(\mu)
\end{equation}
\end{proposition}
\begin{proof}
\cite{hiai2006semicircle}[Theorem 5.5.7, pp. 223] gives us the maximization principle for the Free Poison distribution. Using a change of variables, we get the principle for the Mar\v{c}enko-Pastur distribution. The full details can be found in Appendix~\ref{app:theory}.
\end{proof}

\textbf{The matricial free energy loss function.}
Consider a deep neural network $f_{\theta}: x \in \mathcal{X} \mapsto y \in \mathbb{R}^{d}$. Let $b$ be the batch size and let $\widetilde{Y}$ denote the $d \times b$ sized  batch-code matrix formed by passing as input to the network the inputs $x_1, \ldots, x_b$ organized as:
\begin{equation}
    \widetilde{Y} = \begin{bmatrix}
    f_{\theta}(x_{1}) & f_{\theta}(x_{2}) & \ldots & f_{\theta}(x_{b})
\end{bmatrix}, 
\end{equation}
where for $j = 1, \ldots, b$, $ f_{\theta}(x_{j})$ denotes the output of the network when applied to the input batch of data $\{x_{1}, \ldots x_b\}$. 
If we wish for the network to Gaussianize the inputs, then $\widetilde{Y}$ needs to close, in a spectral sense, to a Gaussian matrix with i.i.d. $\mathcal{N}(0,1)$ entries, for \emph{every} randomly selected batch of inputs.

From Proposition \ref{prop:mandp limit} and \ref{prop:free entropy max}, this implies that the eigenvalues of the matrix $\widetilde{Y} \widetilde{Y}^T/b$ should maximize the discrete analog of the functional in   (\ref{eq:free entropy}). Following this argument, suppose $\{\lambda_{i}\}_{i=1}^{d}$ are the eigenvalues  of $\widetilde{Z} = \widetilde{Y}\widetilde{Y}^\top \in \mathbb{R}^{d \times d}$. Then, the discrete analog of the functional in (\ref{eq:free entropy}) can be obtained by plugging in 
\begin{equation}
    \mu_{\widetilde{Z}} = \dfrac{1}{d} \sum_{i=1}^d \delta_{\lambda_i/d}
\end{equation}
into (\ref{eq:free entropy}). When $d < b$ and $c = d/b$ as in the assumptions for Proposition \ref{prop:mandp limit}, this yields the expression:
\begin{equation} \label{eq:freeloss}
    \begin{split}
        \widehat{\Phi}_c(\widetilde{Y})  &= \left( \frac{1}{d(d-1)}\sum_{i\neq j} \log|\lambda_i/b - \lambda_j/b| \right) \\
        &- \left( \frac{1}{d}\sum_{i=1}^d \left[ \frac{\lambda_i/b}{c} - \left(\frac{1}{c}-1\right)\log(\lambda_i/b) \right] \right) \\
        &= \left( \frac{1}{d(d-1)}\sum_{i\neq j} \log|\lambda_i - \lambda_j| \right) \\
        &- \left( \frac{1}{d}\sum_{i=1}^d \left[ \frac{\lambda_i}{d} - \left(\frac{1}{c}-1\right)\log(\lambda_i) \right] \right) - \frac{b}{d}\log(b),
    \end{split}
\end{equation}
where $\{\lambda_i\}_{i=1}^{d}$ are the  eigenvalues of $\widetilde{Y}\widetilde{Y}^T$. 

Inspecting (\ref{eq:freeloss}), reveals the presence of a constant term $b/d \log b$ on the right hand side that is independent of the $\lambda_i$'s that we are looking to shape or optimize. Eliminating the constant term and substituting $\lambda_i = \sigma_i^2$ where $\sigma_i$ is the $i$-th singular value of $\widetilde{Y}$ and $\lambda_i$ is the corresponding eigenvalue of  $\widetilde{Y}\widetilde{Y}^T$ yields the \emph{free matricial energy}:
\begin{equation} \label{eq:freeloss-normal}
\begin{aligned}
    \overline{\Phi}_c(\widetilde{Y}) 
    = &\frac{1}{d(d-1)} \sum_{i \neq j} \log |\sigma_i^2 - \sigma_j^2| \\
    &- \frac{1}{d} \sum_{i=1}^d \left[ \frac{\sigma_i^2}{d} 
    - \left( \frac{1}{c} - 1 \right) \log(\sigma_i^2) \right].
\end{aligned}
\end{equation}

When $\widetilde{Y}$ is an i.i.d. Gaussian random matrix then, via Proposition \ref{prop:free entropy max}, in the double asymptotic limit of large batch-code matrices we expect it to maximize the matricial free energy function $\overline{\Phi}_c(\widetilde{Y})$. Equivalently,  we  might conclude that $f_{\theta}$ is a Gaussianizing transform if: 
\begin{align}
\theta_{\sf Gaussianizing}
&= \arg\max_{\theta} \, \overline{\Phi}_c(\widetilde{Y}) \\
& = \arg\min_{\theta} \, -\overline{\Phi}_c(\widetilde{Y}), \\
&= \arg\min_{\theta} \, \mathcal{L}_{\text{free}}(\widetilde{Y}), \label{eq:free loss}
\end{align}
where $\mathcal{L}_{free}(\widetilde{Y}) = - \overline{\Phi}_c(\widetilde{Y})$ \textbf{is the newly proposed matricial Free Loss function} which we shall interchangeably refer to as the Free Loss function in what follows. Via the results in \cite{lewis2003mathematics,lewis2005nonsmooth,magnus2019matrix}, the Free Loss function is a differentiable function of the matrix argument when the matrix has distinct singular values. 

\textbf{Characteristics of the Free Loss function.} We note  that the free (energy) loss function $\mathcal{L}_{\textrm{free}}$ discourages the singular values of the batch-code matrix from coalescing or merging into each other via the $\log|\sigma_i^2-\sigma_j^2|$ repulsion term. This project originated from the idea of exploring whether the repulsion term baked into the free energy loss function could mitigate mode collapse while training autoencoders by spreading out the singular value of the batch-code matrix.
 
\textbf{Alternate ways of deriving the Free Loss function.} Note 
that the expression for $\overline{\Phi}_c(\widetilde{Y})$ in (\ref{eq:freeloss-normal}) can be alternately derived by taking the logarithm of  the joint probability distribution of the eigenvalues of a $d \times d$ Wishart random matrix $\widetilde{Y}\widetilde{Y}^T$ as derived simultaneously by \citep{fisher1939sampling,hsu1939distribution,roy1939p} and expressed in matching notation in  \citet{edelman2005random}[pp. 251, Eq. (4.5)]. Omitting the constant terms in the log-likelihood and then converting the resulting expression into a function of the singular values of $\widetilde{Y}$ yields (\ref{eq:freeloss-normal}). Maximizing the log-likelihood like expression for the singular value distribution of $\widetilde{Y}$ that emerges thus is equivalent to finding the maximum likelihood locations of the singular values of an i.i.d. Gaussian random matrix -  the optima are closely connected to the zeros  of the $d$-th degree generalized Laguerre polynomials  as described by \citet{dette2002strong}.  \citet{hiai2006semicircle}[Section 5.5] interpret the matricial free energy function via a large deviation rate function lens.

\begin{figure*}
    \centering
    \subfloat[Scatter plot of training dataset.\label{fig:gmm-data}]{
        \includegraphics[width=0.52\linewidth]{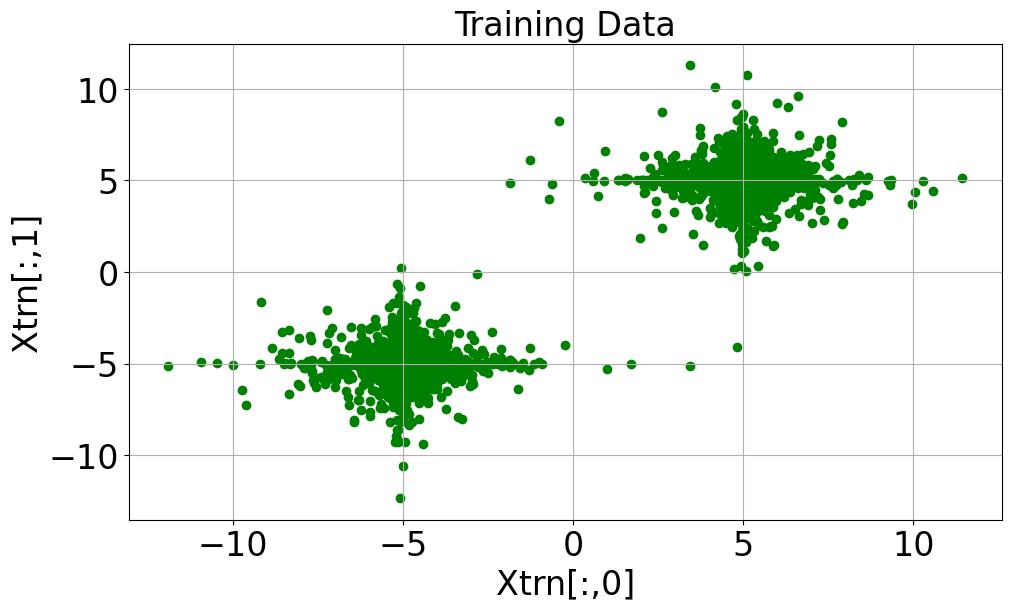}
    }
    \hfill
    \subfloat[Training and test loss curves\label{fig:loss}]{
        \includegraphics[width=0.43\linewidth]{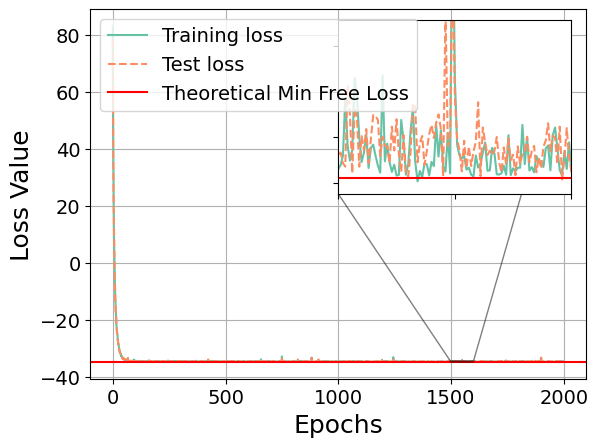}
    }
    \caption{(a) Scatter plot of the training dataset generated as described in Section \ref{sec:free encoder}. 
    (b) The training loss and test curves for a batch of size 256 data points embedded in 32-dimensional space. The blue line is the training loss for a random batch, the orange line is for a random test batch, and the red line is the theoretical minimum computed by sampling a matrix with i.i.d. $\mathcal{N}(0,1)$ entries.}
\end{figure*}

\begin{figure*}[!htbp]
\label{fig:evolve-encoder}
\centering
\begin{subfigure}{\linewidth}
\centering
\begin{subfigure}{0.32\linewidth}
\includegraphics[width=\linewidth]{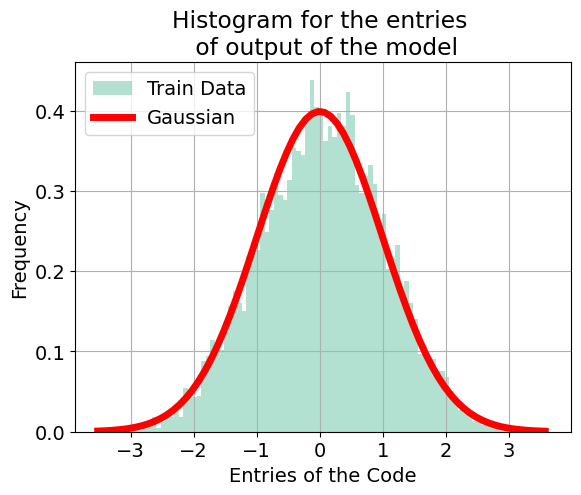}
\caption{Empirical pdf of the entries vs pdf of $\mathcal{N}(0,1)$.}
\label{fig:hist-2000-train}
\end{subfigure}
\begin{subfigure}{0.32\linewidth}
\includegraphics[width=\linewidth]{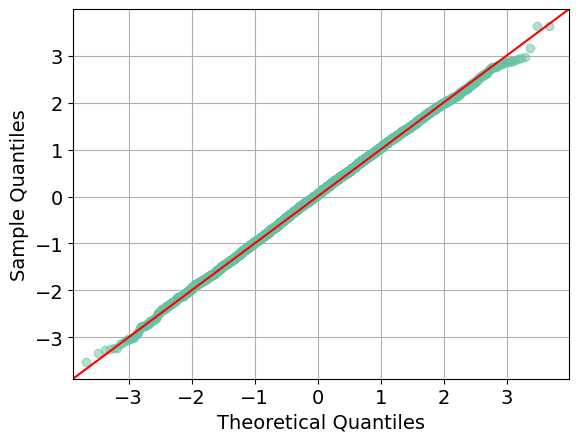}
\caption{Q-Q plot relative to a standard $\mathcal{N}(0,1)$ distribution.}
\label{fig:qq-2000-train}
\end{subfigure}
\begin{subfigure}{0.32\linewidth}
\includegraphics[width=\linewidth]{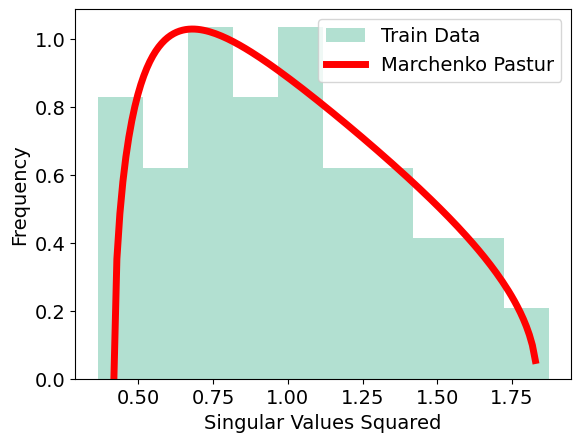}
\caption{Distribution of sample covariance matrix eigenvalues.}
\label{fig:svd-2000-train}
\end{subfigure}
\caption{Scalar and Matricial quantities for the code $\mathcal{E}_{\mathrm{opt}}(X_b^{\text{train}})$ for a batch of train data $X_b^{\text{train}}$.}
\label{fig:train-2000}
\end{subfigure}
\begin{subfigure}{\linewidth}
\centering
\begin{subfigure}{0.32\linewidth}
\includegraphics[width=\linewidth]{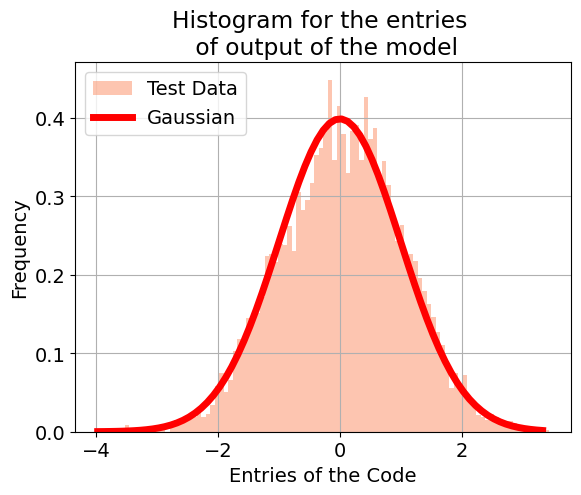}
\caption{Empirical pdf of the entries vs pdf of $\mathcal{N}(0,1)$.}
\label{fig:hist-2000-test}
\end{subfigure}
\begin{subfigure}{0.32\linewidth}
\includegraphics[width=\linewidth]{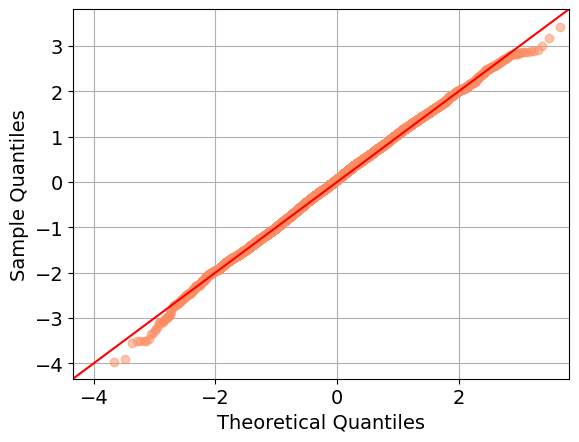}
\caption{Q-Q plot relative to a standard $\mathcal{N}(0,1)$ distribution.}
\label{fig:qq-2000-test}
\end{subfigure}
\begin{subfigure}{0.32\linewidth}
\includegraphics[width=\linewidth]{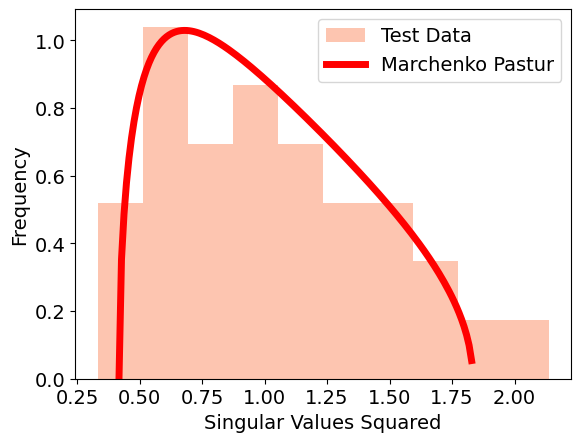}
\caption{Distribution of sample covariance matrix eigenvalues.}
\label{fig:svd-2000-test}
\end{subfigure}
\caption{Scalar and Matricial quantities for the code $\mathcal{E}_{\mathrm{opt}}(X_b^{\text{test}})$ for a batch of test data $X_b^{\text{test}}$.}
\label{fig:test-2000}
\end{subfigure}
\begin{subfigure}{\linewidth}
\centering
\begin{subfigure}{0.32\linewidth}
\includegraphics[width=\linewidth]{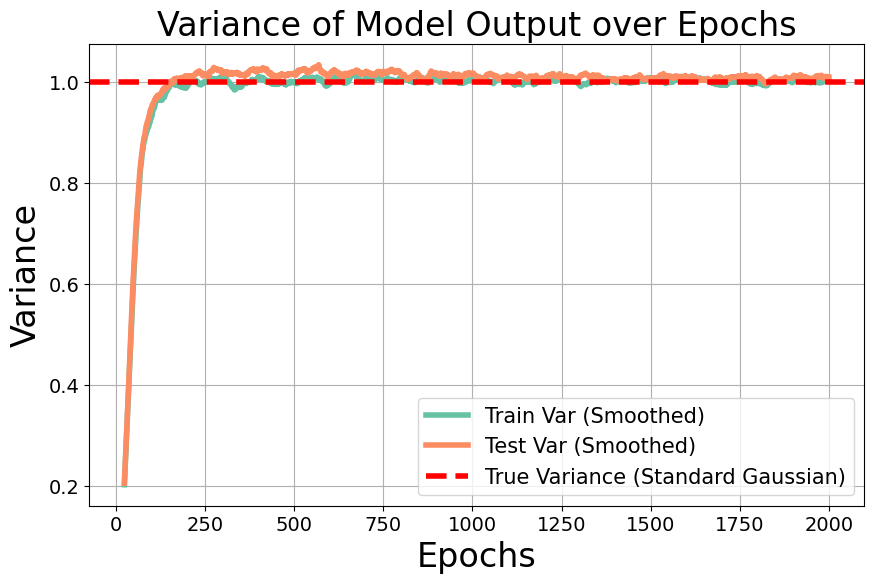}
\caption{Sample Variance.}
\label{fig:var}
\end{subfigure}
\begin{subfigure}{0.32\linewidth}
\includegraphics[width=\linewidth]{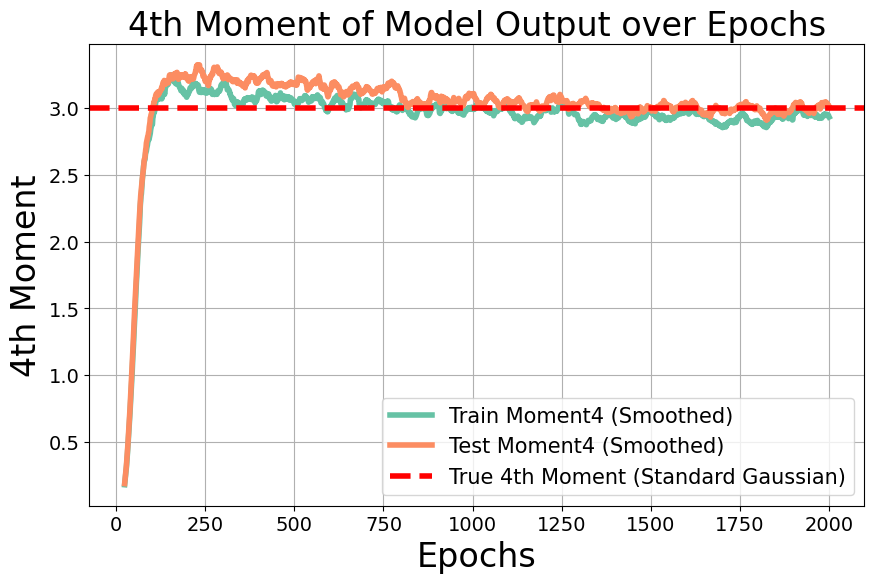}
\caption{Sample Fourth Moment.}
\label{fig:moment4}
\end{subfigure}
\begin{subfigure}{0.32\linewidth}
\includegraphics[width=\linewidth]{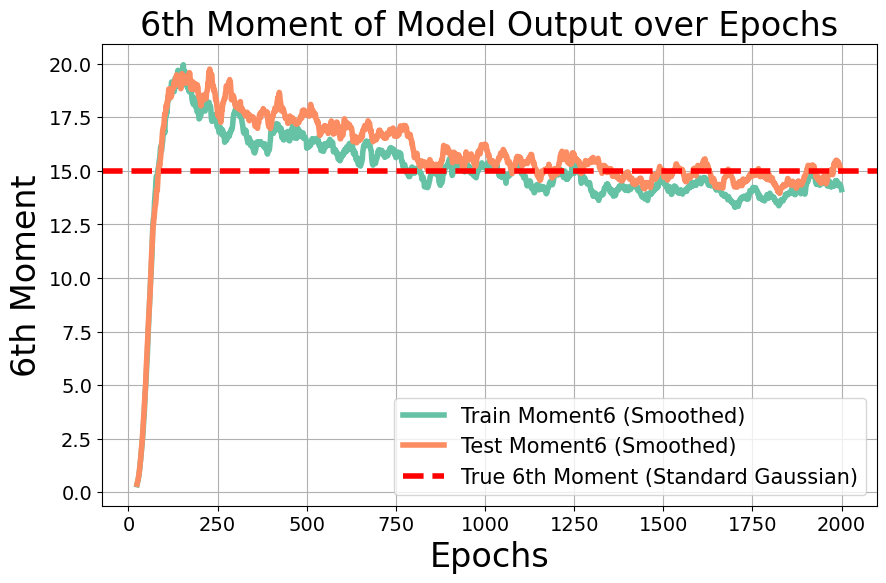}
\caption{Sample Sixth Moment.}
\label{fig:moment6}
\end{subfigure}
\caption{Moment Matching for entries of $\mathcal{E}_{\mathrm{opt}}(X_b^{\text{train}}), \mathcal{E}_{\mathrm{opt}}(X_b^{\text{test}})$.}
\label{fig:moments}
\end{subfigure}
\caption{Visualization of the outputs of a free Gaussianizing encoder trained as  described in Section \ref{sec:free encoder}. The top two rows show the histogram of the entries of entries of $\mathcal{E}_{\mathrm{opt}}(X_b)$, a Q-Q (or quantile-quantile) plot for the entries of $\mathcal{E}_{\mathrm{opt}}(X_b)$ which  compares the quantiles of the empirical data against the quantiles of a theoretical standard normal distribution, and the empirical eigenvalue disribution  of $(1/b),\mathcal{E}_{\mathrm{opt}}(X_b)\mathcal{E}_{\mathrm{opt}}(X_b)^\top$ for the  training and test dataset relative to the   Mar\v{c}enko-Pastur distribution in (\ref{eq:mandp}) with parameter $c = d/b = 32/256$. 
Row 3 shows the entrywise variance, fourth, and sixth moments (red line is true moments for $\mathcal{N}(0,1)$ target) of the training and test data as the free Gaussianizing encoder is trained.}
\label{fig:2}
\end{figure*}

\section{Free Gaussianizing Encoder}\label{sec:free encoder}
We first examine whether training an encoder to minimize Free Loss produces Gaussian codes.  Let $ \{\mathcal{E}_{\theta}: x \in X \mapsto \mathbb{R}^d\}$ denote an encoder. We train a free Gaussianizing encoder, using the Free Loss $\mathcal{L}_{\textrm{free}}$ as the loss function and mini-batched Adam as the optimizer. We note that the mini-batch used at each iteration is randomly selected. Throughout, we let $X_b$ be a mini-batch with $b$ data points.

\begin{table*}[ht]
\centering
\begin{tabular}{l|cccc|ccc|c}
\toprule
\textbf{Metric} & \multicolumn{4}{c}{Image} & \multicolumn{3}{c}{Audio} & Text \\
& MNIST & CIFAR10 & CelebA & Imagenet & GTZAN & ESC50 & Urbansound & IMDB \\
\midrule
KS & 0.0190 & 0.0219 & 0.0378 & 0.0174 & 0.0120 & 0.0255 & 0.0750 & 0.0235 \\
$\Delta_{\texttt{OT}}$ & 0.0168 & 0.0154 & 0.0278 & 0.0134 & 0.0172 & 0.0311 & 0.0242 & 0.0103 \\
Free Loss & 0.0024 & 0.0012 & 0.0012 & 0.0063 & 0.00009 & 0.0069 & 0.1007 & 0.0059\\
\bottomrule
\end{tabular}
\caption{Deviation from Gaussianity statistics on test data for image, audio, and text data. We have the scalar Kolmogorov-Smirnov statistic, the vectorial relative error in the optimal transport cost, and matricial relative error in Free Loss. The image and text datasets were trained with a batch size of $b = 128$ while the audio dataset trained using a batch size of $b = 64$. We employed $ d= 32$ dimensional embeddings for all datasets}
\label{tab:encoder-real-delta-metrics}
\end{table*}

\textbf{Data.} Let $n \in \mathbb{N}$ be the number of samples of training data. Let $p$ be the dimensionality of the training sample $x_i$. Let  $\mu\in\mathbb{R}^p$ be a mean vector, we define the input data as follows. 
Let $s\in\{\pm1\}^{2n}$ be a balanced label vector with $\sum_{i=1}^{2n}s_i=0$.
For each sample $i=1,\dots,2n$, draw a base vector $u_i\in\mathbb{R}^p$ whose coordinates are i.i.d.\ $\chi^2_1$, where $\chi^2_1 = |\mathcal{N}(0,1)|^2$ is the chi-squared distribution. We generate $n = 2560$ training samples and the same number of test data samples via the construction: 
\[
    x_i \;=\, 0.5 \, u_i \;+\; s_i\,\mu,
\]
by setting $p = 2$ and $\mu = \begin{bmatrix} 5 & 5  \end{bmatrix}^T$, Figure~\ref{fig:gmm-data} shows the samples of the training data set that we shall use for all the simulation in this paper. 

\textbf{Network.} We then train a four layer fully connected MLP with $\tanh$ activation to learn a $d = 32$ dimensional embedding with a batch size of $b = 256$.  We train the network\footnote{The exact architecture is $\mathrm{Linear}(2,32)\!\to\!\mathrm{Linear}(32,32)\!\to\!\tanh\!\to\!\mathrm{Linear}(32,32)\!\to\!\tanh\!\to\!\mathrm{Linear}(32,32)\!\to\!\tanh\!\to\!\mathrm{Linear}(32,32)$.} for 2000 epochs, using mini-batched Adam with a learning rate of $10^{-3}.$

\textbf{Deviation from Gaussianity Statistics.} To confirm that the output is Gaussian, we compute a variety of different deviation-from-Gaussianity metrics. In particular, we consider the following statistics. 

\begin{enumerate}[nosep, leftmargin=*,itemsep = 2pt]
    \item \textbf{Scalar:} We flatten $\mathcal{E}_{\mathrm{opt}}(X_b)$ into a $db$-dimensional vector, $X_b$ is a batch of data. To check if the entries are from $\mathcal{N}(0,1)$ by plotting the histogram and the qq-plot for the entries. Numerically, we also compute the Kolmogorov-Smirnov statistic.
    \item \textbf{Vectorial:} We consider each column of the $d \times b$ matrix output $\mathcal{E}_{\mathrm{opt}}(X_b)$ as a sample of a distribution in $\mathbb{R}^d$. Then to verify, whether we have Gaussian samples, we compute the relative excess optimal transport cost 
    \begin{equation} \label{eq:ot}
       \Delta_{\texttt{OT}} :=  \frac{\left|\texttt{OT}(\mathcal{E}_{\mathrm{opt}}(X_b), G) - \mathbb{E}\left[\texttt{OT}(G, \tilde{G})\right]\right|}{\mathbb{E}\left[\texttt{OT}(G, \tilde{G})\right]}
    \end{equation}
    where $G, \tilde{G}$ are matrices with IID $\mathcal{N}(0,1)$ entries and $\texttt{OT}(A,B)$ is the discrete optimal transport cost 
    \[
        \texttt{OT}(A,B)\ :=\ \min_{\pi\in S_b}\ \frac{1}{b}\sum_{j=1}^b\|a_j-b_{\pi(j)}\|_2^2,
    \]
    where $S_b$ is the set of permutations
    \item \textbf{Matricial:} Finally, we consider the matrix verification of Gaussianity. In particular, we compute the singular values of $\mathcal{E}_{\mathrm{opt}}(X_b)$, and verify that the distribution is the Mar\v{c}enko-Pastur distribution. 
\end{enumerate}

\textbf{Results.} Figure \ref{fig:loss} shows the Free Loss $\mathcal{L}_{free}$ on training and test data during training. As the figure shows, we see that after about 50 epochs, we see that the loss on both training and test data, is within 1\% of the theoretical minimum Free Loss (red dotted line). 

Figures~\ref{fig:hist-2000-train},~\ref{fig:hist-2000-test},~\ref{fig:qq-2000-train}, and~\ref{fig:qq-2000-test}, show the histogram and QQ plot. Here we see that both metrics are matched perfectly. Finally, Figures~\ref{fig:svd-2000-train} and~\ref{fig:svd-2000-test} shows the singular value squared distribution and that it matches the Mar\v{c}enko-Pastur distribution. Hence we see that we have successfully Gaussianized the data. Appendix~\ref{app:batchsize} explores the effect of the batch size. 

\begin{remark}[Training Length] We note that while the we hit the minimum Free Loss relatively quickly (around epoch 75), we have to continue training well past this (to epoch 2000) to produce nearly Gausssian outputs. This can be seen from Figure~\ref{fig:moments}, where we see that we do not match the  higher order moments of the normal distribution until much later in training. More plots for the training dynamics can be found in Appendix~\ref{app:dynamics}. 
\end{remark} 

\textbf{Real Data.} We train Free Loss minimizing encoders on real data. We do this for image data -- MNIST \citep{6296535}, CIFAR10 \citep{Krizhevsky2009LearningML}, CelebA \citep{liu2015faceattributes}, and a subset of Imagenet \citep{5206848}, text -- IMDB movie reviews \citep{maas-EtAl:2011:ACL-HLT2011}, and audio -- Enviromental Sounds 50 \citep{piczak2015dataset}, Urbansound \citep{Salamon:UrbanSound:ACMMM:14}, and GTZAN \citep{Tzanetakis2001AutomaticMG}. Table~\ref{tab:autoencoder-delta-metrics} shows that in all cases, we can Gaussianize. 
The relevant plots can be seen in Appendix~\ref{app:real}.

\begin{figure*}[!htbp]
\centering
\begin{subfigure}{0.45\linewidth}
\includegraphics[width=\linewidth]{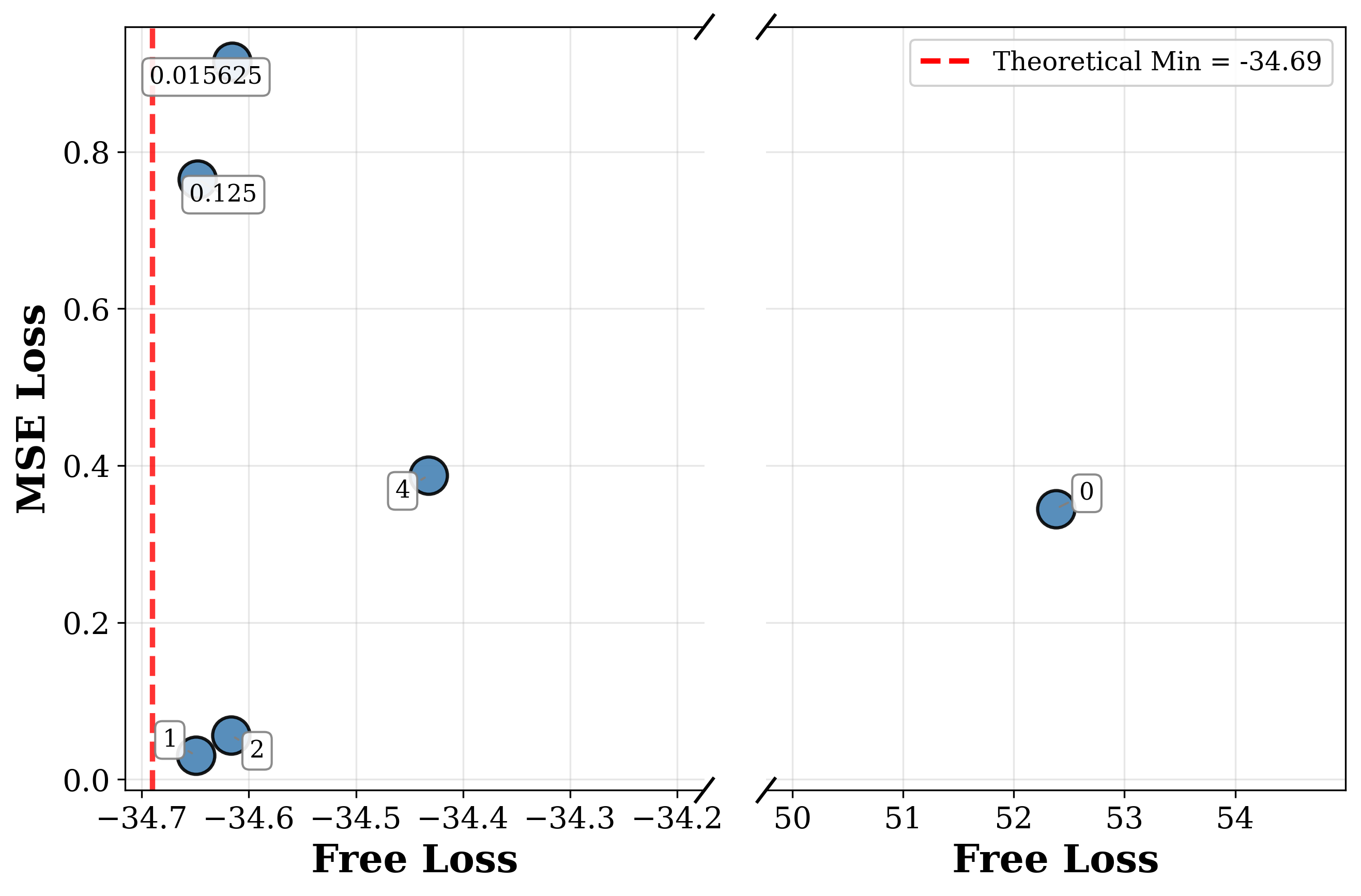}
\caption{Free Loss regularized loss (Eq.~\ref{eq:free-autoencoder}).}
\end{subfigure}
\quad \quad
\begin{subfigure}{0.45\linewidth}
\includegraphics[width=\linewidth]{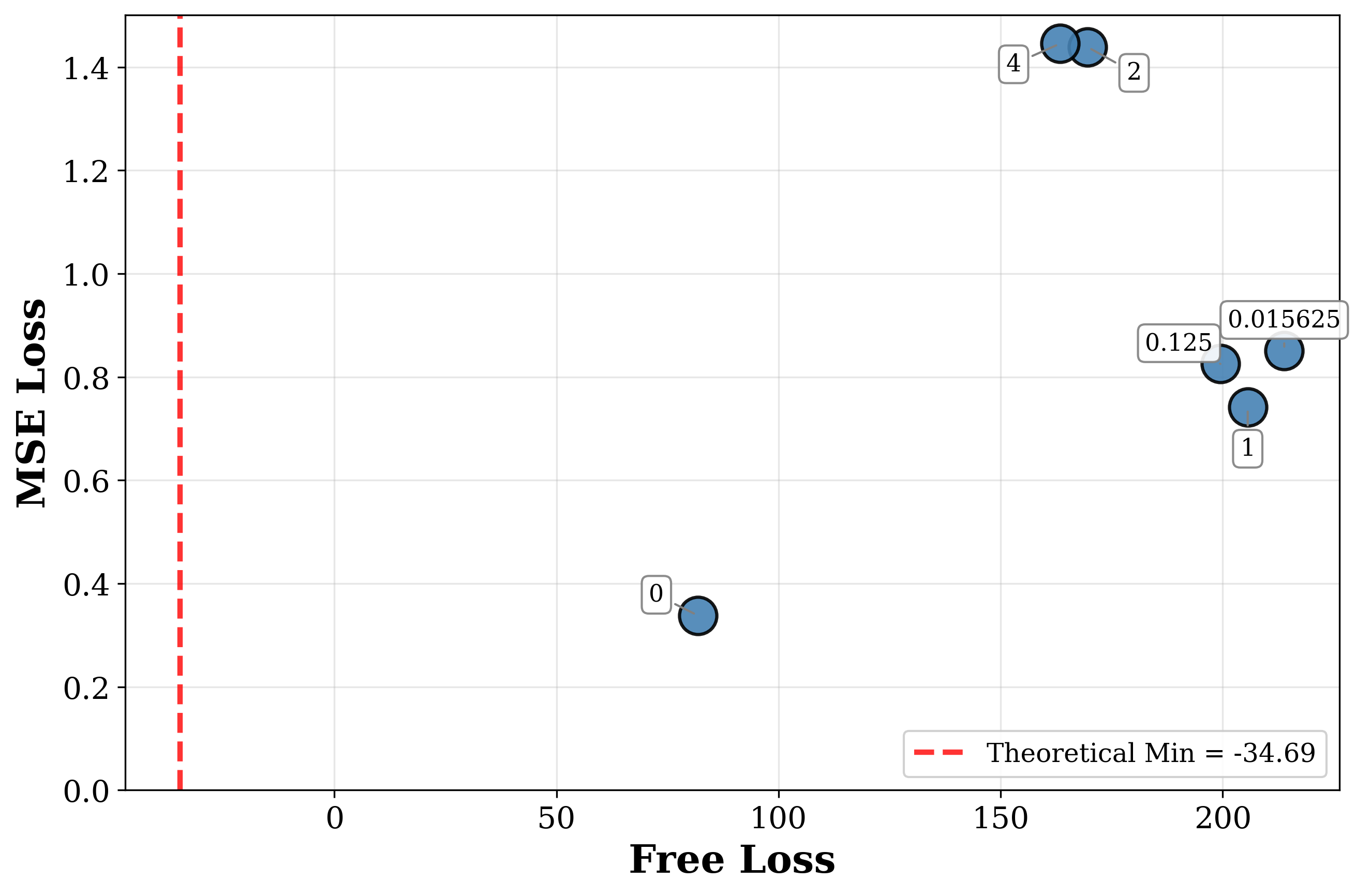}
\caption{Tikhonov regularized loss (Eq.~\ref{eq:tikhonov-autoencoder}).}
\end{subfigure}
\caption{Figure showing the Free Loss vs MSE for training an autoencoder for different values of $\tau$ as described in Section~\ref{sec:free autoencoder}. Here the theoretical minimum Free Loss depicted at $-34.69$ is  the value obtained by averaging the empirical Free Loss of many independent  $32 \times 256$ i.i.d. Gaussian matrices.}
\label{fig:free-autoencoder-L-tau}
\end{figure*}

\begin{figure*}[t]
    \centering
    \subfloat[Tikhonov-regularized autoencoder ($\tau=1$)]{%
        \includegraphics[width=0.32\linewidth]{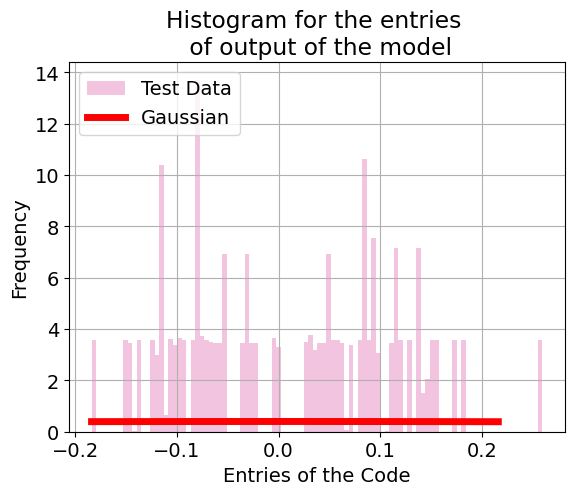}
    }
    \subfloat[Unregularized autoencoder ($\tau=0$)]{%
        \includegraphics[width=0.32\linewidth]{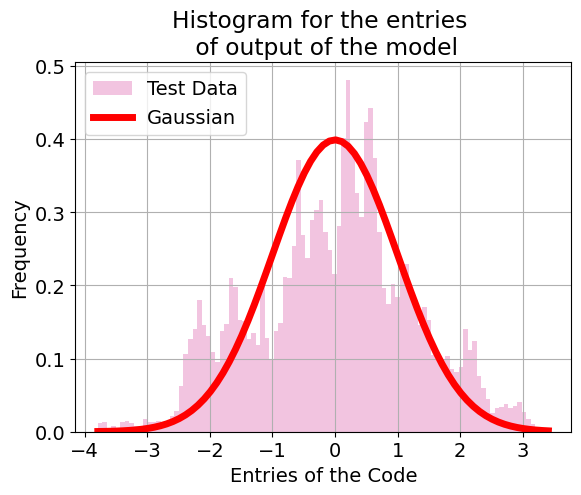}
    }
    \subfloat[Free-loss–regularized autoencoder ($\tau=1$)]{%
        \includegraphics[width=0.32\linewidth]{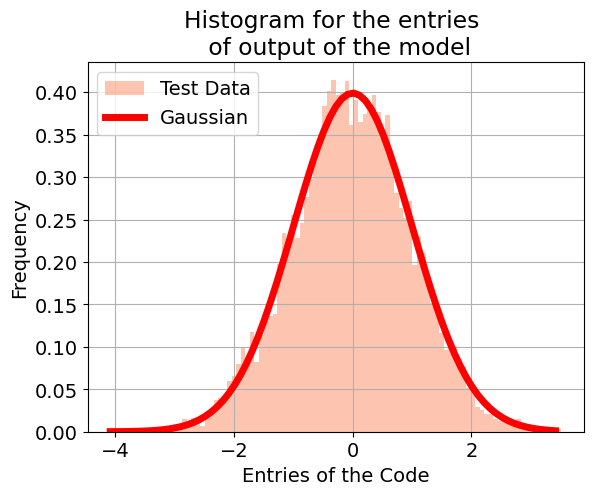}
    }
    \caption{Histograms of autoencoder code entries under three training regimes as described in Section~\ref{sec:free autoencoder}: Tikhonov regularization ($\tau=1$), no regularization ($\tau=0$), and free-loss regularization ($\tau=1$), evaluated on the test set after 2000 epochs.}
    \label{fig:autoencoder-firstcol}
\end{figure*}

The encoder models consist of specialized architectures tailored to each data modality to produce a 32-dimensional embedding. For audio, a Conformer was utilized \citep{gulati2020conformer}. Text inputs were encoded using a standard Transformer encoder architecture \citep{vaswani2017attention}, aggregating the sequence representations through attention pooling to yield a fixed dimensional embedding. Image encoding leveraged an EfficientViT backbone (efficientvit m2 variant) without pretraining, followed by a linear projection from its 100-dimensional features to the 32-dimensional target space \citep{liu2023efficientvit}. These encoders facilitate multimodal representation learning by mapping diverse inputs to a unified latent space.

\section{Free Gaussianizing Autoencoder}\label{sec:free autoencoder}
Let $ \{\mathcal{E}_{\theta}: x \in X \mapsto \mathbb{R}^d\}$ be the encoder and  $ \{\mathcal{D}_{\gamma}: \mathbb{R}^d  \mapsto x \in X\}$ be the decoder. We train a Free Gaussianizing encoder $\mathcal{E}_{\sf opt}$ and decoder $\mathcal{D}_{\sf opt}$ by training to minimize the loss function:
\begin{align}\label{eq:free-autoencoder}
    \frac{1}{b}\left\|\mathcal{D}_{\gamma}(\mathcal{E}_{\theta}(X_b)) - X_b\right\|_F^2 + \tau \, \mathcal{L}_{\sf free}(\mathcal{E}_{\theta}(X_b))
\end{align}
where $\mathcal{E}_{\theta}(X_b)$ is the $d \times b$ batch-code matrix that we are looking to Gaussianize and $\| \cdot \|_F$ is the Frobenius norm of the matrix argument. Here we use the standard mean squared error loss (MSE) and we use Free Loss $\mathcal{L}_{\sf free}$ as a regularizer with regularization strength $\tau$. 

To illustrate the value of the matricial Free Loss we compare this to a Tikhonov regularized trained trained to minimize the loss function:
\begin{align} \label{eq:tikhonov-autoencoder}
    \frac{1}{b}\left\|\mathcal{D}_{\gamma}(\mathcal{E}_{\theta}(X_b)) - X_b \right\|_F^2 + \tau \left\| \mathcal{E}_{\theta}(X_b) \right\|_F^2.
\end{align}
Let $\mathcal{D}_{\sf Tikh}, \mathcal{E}_{\sf Tikh}$ be the networks trained using the Tikhonov regularized loss and note that $\left\| \mathcal{E}(X_b) \right\|_F^2$ is the negative log likelihood for a Gaussian distribution. 

We begin by doing a sweep of $\tau$ for both, the Free Loss regularized autoencoder, and Tikhonov regularized autoenconder. This can be seen in Figure~\ref{fig:free-autoencoder-L-tau}. Here we can see that for the Free Loss regularized autoencoder, we can simultaneously optimize both the reconstruction MSE and Free Loss. Figure~\ref{fig:free-autoencoder-L-tau} (Left) we see that for $\tau = 1$, we are close to the optimal Free Loss and MSE loss. On the other hand, we see that for the Tikhonov regularized autoencoder for $\tau > 0$, we obtain suboptimal MSE and non-Gaussian codes as evidenced by the corresponding sub-optimal Free Losses. 

Next, we compare the Tikhonov regularized autoencoder ($\tau = 1$), the Free Loss regularized autoencoder ($\tau = 1$), and the unregularized autoencoder $(\tau = 0)$. Table~\ref{tab:autoencoder-delta-metrics} shows some of deviation-from-Gaussianization metrics. Here we see that the Tikhonov regularized and the unregularized autoencoder have codes that are far from Gaussian. Finally, we see that the Free Loss regularized autoencoder, results in Gaussian code.  More plots can be found in Appendix~\ref{app:autoencoder}.

\begin{table*}[ht]
\centering
\begin{tabular}{lccccc}
\toprule
\textbf{Method} &  KS & $\Delta_{\texttt{OT}}$ & MSE & Rel. Err. Free Loss & Rel. Err. 8th Moment \\
\midrule
Unregularized & $0.23 \pm 0.08$ & $0.713 \pm 0.294$ & $0.56 \pm 0.13$ & $4.26 \pm 0.86$ & $24 \pm 38$\\
Tikhonov  & $0.49 \pm 0.02$ & $0.049 \pm 0.006$ & $0.72 \pm 0.13$ & $1.19 \pm 0.48$ & $1 \pm 0$ \\
Free Loss  & $\mathbf{0.03 \pm 0.01}$ & $\mathbf{0.039 \pm 0.010}$ & $\mathbf{0.18 \pm 0.29}$ &  $\mathbf{0.004 \pm 0.005}$ & $\mathbf{0.16 \pm 0.15}$ \\
\bottomrule
\end{tabular}
\caption{\textbf{Deviation-from-Gaussianization metrics on test codes (mean $\pm$ std over 10 trials).} \emph{KS} is the Kolmogorov–Smirnov statistic comparing standardized code entries to $\mathcal{N}(0,1)$. $\Delta_{\texttt{OT}}$ is Equation~\ref{eq:ot}; \emph{MSE} is the reconstruction mean-squared error. \emph{Rel. Err. Free Loss} is the relative deviation of the matricial free-energy objective from its Gaussian reference, and \emph{Rel. Err. 8th Moment} is the relative error of the empirical 8th central moment of code entries from the Gaussian value $105$; lower is better for all metrics. Here relative error is $RelErr(T)=|(Tref-T)/Tref|.$ The best metric is in bold.}
\label{tab:autoencoder-delta-metrics}
\end{table*}

\subsection{Application: Solving Inverse problems}\label{sec:inverse problems}
Suppose we are given measurements of the form $z = \mathcal{A}(x)$.
If we have a Gaussianizing encoder  $\mathcal{E}_{\sf opt}$ then $\mathcal{E}_{\sf opt}(x)$ should ideally be an i.i.d. Gaussian vector with zero mean and unit variance entries. We can utilize this Gaussian vectorial prior to formulate the recovery problem in data space as minimizing the following:
\begin{equation} \label{eq:recover-free}
    \|z - \mathcal{A}\left(\mathcal{D}_{\gamma}(\mathcal{E}_{\theta}(x))\right) \|_2^2 + \rho\|\mathcal{E}_{\theta}(x)\|^2.
\end{equation}

\begin{figure}[htp]
    \centering
    \includegraphics[width=0.75\linewidth]{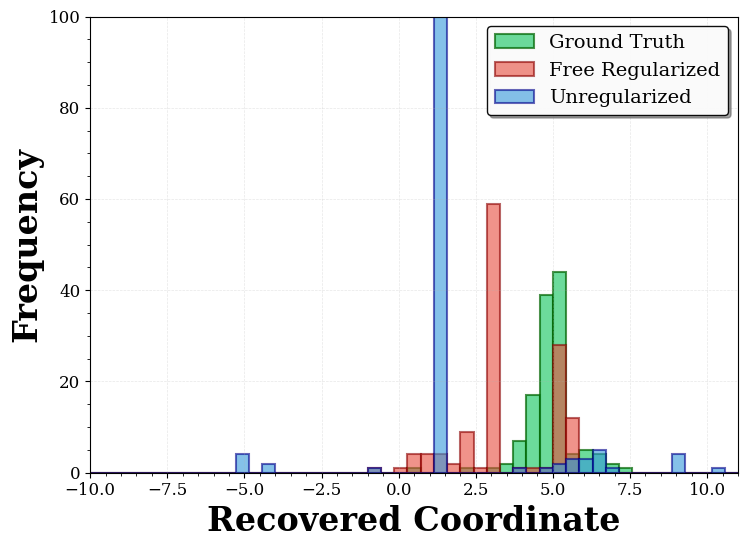}
    \caption{Histogram of the recovered coordinate conditioned on the sign of the projected coordinate $z$ for the free regularized autoencoder (red) versus  the Tikhonov regularized (classical) autoencoder (blue)  for the setting where $z$ in (\ref{eq:recover-free}) is positive. The ground truth is plotted in blue.  We used $\rho = 0.0005$ and performed 5000 gradient steps using SGD with a learning rate of $10^{-3}$ to optimize \eqref{eq:recover-free} for $x$ in (\ref{eq:recover-free}). } 
    \label{fig:recon}
\end{figure}

Figure~\ref{fig:recon} shows the results of the experiment where $\mathcal{A}(x) = P x$ for $P =  \begin{bmatrix}1 & 0  \end{bmatrix}$.  Here we see that the free regularized autoencoder, better recovers the missing coordinate, while preserving the given coordinate. In particular, we see the Free Regularized autoencoder has a MSE of $1.4$ for the given coordinate, and $5.6$ for the missing one. While the unregularized autoencoder has an MSE of $18.5$ for the given coordinate and 10.4 for the missing coordinate. We used 5120 training points, 256 test points, and embedding dimension $d = 32$, a batch size $b =  256$ and our four layer encoder and decoders as before. The encoder and decoder were pre-trained for 2000 epochs.

\section{Future Work}
Our matricial free energy loss function provides a principled route to Gaussianizing autoencoder codes. Several directions remain open. One is a sharper theoretical analysis of the induced optimization landscape, including convergence behavior, local minima, fundamental limits on Gaussianization and how architectural choices shape spectral dynamics and the quality of the resulting. Another is to quantify the role of batch size on the approximation performance, since it sets the empirical spectrum of the batch-code matrix warrants careful study. Beyond, representation learning, the framework could be extended to generative models such as VAEs or flows, enabling Gaussianized latent sampling and testing its impact on sample quality and diversity. It would also be natural to explore new loss designs by combining the matricial free energy with perceptual, contrastive, class-conditional or optimal transport -based objectives to strengthen robustness across modalities.   Gaussianizing latent codes can be useful for distribution shift detection because they allow us to leverage classical multivariate theory for deviation from Gaussianity in code space as in \citep{ebner2020tests,henze2002invariant,kay1993fundamentals}.

\section{Conclusions}
Using free probability and random matrix theory, we developed a regularization framework for autoencoders. Our method uses a matricial free energy-based loss function to encourage latent codes to mimic the spectral properties of Gaussian random matrices. This yields statistically robust, Gaussian-like representations that generalize well. We used this to create a free Gaussianizing autoencoder that minimizes reconstruction loss while producing Gaussian latent codes and demonstrate its potential utility in solving ill-posed inverse problems. The work connects deep learning and spectral theory, with future research areas including structured data and other unsupervised paradigms. 

\subsection*{Acknowledgment}

R. N is supported by AFOSR award  FA9550-25-1-0377.

\newpage

\bibliographystyle{plainnat}
\bibliography{references}

\clearpage
\newpage

 \clearpage
 \newpage

\appendix


\newcommand{\AEimg}[3]{%
  \IfFileExists{#1/#2-autoencer-#3}{%
    \includegraphics[width=\linewidth]{#1/#2-autoencer-#3}%
  }{%
    \includegraphics[width=\linewidth]{#1/#2-autoencoder-#3}%
  }%
}

\newcommand{\EncImg}[2]{%
  \IfFileExists{images/#1-#2}{%
    \includegraphics[width=\linewidth]{images/#1-#2}%
  }{%
    \includegraphics[width=\linewidth]{images/#2}%
  }%
}

\newcommand{\MomentImg}[2]{%
  \IfFileExists{new_auto_images/#1-#2.png}{%
    \includegraphics[width=\linewidth]{new_auto_images/#1-#2.png}%
  }{%
    \includegraphics[width=\linewidth]{new_auto_images/#2.png}%
  }%
}

\newcommand{\MakeEncoderEvolveFig}[4]{%
\begin{figure*}[!htbp]
\centering

\begin{subfigure}{\linewidth}
\centering
  \begin{subfigure}{0.32\linewidth}
    \EncImg{#1}{hist-epoch-#3-train.png}
    \caption{Empirical pdf of the entries vs.\ $\mathcal{N}(0,1)$.}
    \label{fig:enc-#1-hist-#3-train}
  \end{subfigure}
  \begin{subfigure}{0.32\linewidth}
    \EncImg{#1}{qq-train-epoch-#3-train.png}
    \caption{Q--Q plot vs.\ $\mathcal{N}(0,1)$.}
    \label{fig:enc-#1-qq-#3-train}
  \end{subfigure}
  \begin{subfigure}{0.32\linewidth}
    \EncImg{#1}{singular-vals-train-epoch-#3-train.png}
    \caption{Eigenvalue distribution of the sample covariance.}
    \label{fig:enc-#1-svd-#3-train}
  \end{subfigure}
  \caption{Scalar and matricial quantities for $\mathcal{E}_{\mathrm{opt}}(X_b^{\text{train}})$.}
  \label{fig:enc-#1-train-#3}
\end{subfigure}

\begin{subfigure}{\linewidth}
\centering
  \begin{subfigure}{0.32\linewidth}
    \EncImg{#1}{hist-epoch-#3-test.png}
    \caption{Empirical pdf of the entries vs.\ $\mathcal{N}(0,1)$.}
    \label{fig:enc-#1-hist-#3-test}
  \end{subfigure}
  \begin{subfigure}{0.32\linewidth}
    \EncImg{#1}{qq-test-epoch-#3-test.png}
    \caption{Q--Q plot vs.\ $\mathcal{N}(0,1)$.}
    \label{fig:enc-#1-qq-#3-test}
  \end{subfigure}
  \begin{subfigure}{0.32\linewidth}
    \EncImg{#1}{singular-vals-test-epoch-#3-test.png}
    \caption{Eigenvalue distribution of the sample covariance.}
    \label{fig:enc-#1-svd-#3-test}
  \end{subfigure}
  \caption{Scalar and matricial quantities for $\mathcal{E}_{\mathrm{opt}}(X_b^{\text{test}})$.}
  \label{fig:enc-#1-test-#3}
\end{subfigure}

\begin{subfigure}{\linewidth}
\centering
  \begin{subfigure}{0.32\linewidth}
    \MomentImg{#1}{var}
    \caption{Sample variance (train/test over time).}
    \label{fig:enc-#1-var}
  \end{subfigure}
  \begin{subfigure}{0.32\linewidth}
    \MomentImg{#1}{moment4}
    \caption{Sample fourth moment.}
    \label{fig:enc-#1-m4}
  \end{subfigure}
  \begin{subfigure}{0.32\linewidth}
    \MomentImg{#1}{moment6}
    \caption{Sample sixth moment.}
    \label{fig:enc-#1-m6}
  \end{subfigure}
  \caption{Moment matching for entries of $\mathcal{E}_{\mathrm{opt}}(X_b^{\text{train}})$ and $\mathcal{E}_{\mathrm{opt}}(X_b^{\text{test}})$.}
  \label{fig:enc-#1-moments}
\end{subfigure}

\caption{Visualization of outputs for a free Gaussianizing \textbf{encoder} (#2) at epoch #3. Top two rows: histogram, Q--Q plot, and eigenvalue distribution with Mar\v{c}enko--Pastur overlay (shape #4). Bottom row: variance, fourth, and sixth moments (red line = $\mathcal{N}(0,1)$ target). The higher moments, are affected by outlier entries, we believe changing the architecture can fix this. See Section~\ref{sec:real-plots} for more details about the experimental setup. }
\label{fig:evolve-encoder-#1-#3}
\end{figure*}
}

\newcommand{\MakeAEEvolveFig}[5]{%
\begin{figure*}[!htbp]
\centering

\begin{subfigure}{\linewidth}
\centering
  \begin{subfigure}{0.32\linewidth}
    \AEimg{#1}{#2}{hist-epoch-#4-train.png}
    \caption{Empirical pdf of the entries vs.\ $\mathcal{N}(0,1)$.}
    \label{fig:#2-hist-#4-train}
  \end{subfigure}
  \begin{subfigure}{0.32\linewidth}
    \AEimg{#1}{#2}{qq-train-epoch-#4-train.png}
    \caption{Q--Q plot vs.\ $\mathcal{N}(0,1)$.}
    \label{fig:#2-qq-#4-train}
  \end{subfigure}
  \begin{subfigure}{0.32\linewidth}
    \AEimg{#1}{#2}{singular-vals-train-epoch-#4-train.png}
    \caption{Eigenvalue distribution of the sample covariance.}
    \label{fig:#2-svd-#4-train}
  \end{subfigure}
  \caption{Scalar and matricial quantities for the code of \textbf{#3} on $X_b^{\text{train}}$.}
  \label{fig:#2-train-#4}
\end{subfigure}

\begin{subfigure}{\linewidth}
\centering
  \begin{subfigure}{0.32\linewidth}
    \AEimg{#1}{#2}{hist-epoch-#4-test.png}
    \caption{Empirical pdf of the entries vs.\ $\mathcal{N}(0,1)$.}
    \label{fig:#2-hist-#4-test}
  \end{subfigure}
  \begin{subfigure}{0.32\linewidth}
    \AEimg{#1}{#2}{qq-test-epoch-#4-test.png}
    \caption{Q--Q plot vs.\ $\mathcal{N}(0,1)$.}
    \label{fig:#2-qq-#4-test}
  \end{subfigure}
  \begin{subfigure}{0.32\linewidth}
    \AEimg{#1}{#2}{singular-vals-test-epoch-#4-test.png}
    \caption{Eigenvalue distribution of the sample covariance.}
    \label{fig:#2-svd-#4-test}
  \end{subfigure}
  \caption{Scalar and matricial quantities for the code of \textbf{#3} on $X_b^{\text{test}}$.}
  \label{fig:#2-test-#4}
\end{subfigure}

\begin{subfigure}{\linewidth}
\centering
  \begin{subfigure}{0.32\linewidth}
    \MomentImg{#2}{var}
    \caption{Sample variance (train/test over time).}
    \label{fig:#2-var}
  \end{subfigure}
  \begin{subfigure}{0.32\linewidth}
    \MomentImg{#2}{moment4}
    \caption{Sample fourth moment.}
    \label{fig:#2-m4}
  \end{subfigure}
  \begin{subfigure}{0.32\linewidth}
    \MomentImg{#2}{moment6}
    \caption{Sample sixth moment.}
    \label{fig:#2-m6}
  \end{subfigure}
  \caption{Moment matching for entries of the code on $X_b^{\text{train}}$ and $X_b^{\text{test}}$.}
  \label{fig:#2-moments}
\end{subfigure}

\caption{Visualization of outputs for the \textbf{#3 autoencoder} at epoch #4. Top two rows: histogram, Q--Q plot, and eigenvalue distribution with Mar\v{c}enko--Pastur overlay (shape #5). Bottom row: variance, fourth, and sixth moments (red line = $\mathcal{N}(0,1)$ target). The higher moments, are affected by outlier entries, we believe changing the architecture can fix this. See Section~\ref{sec:E2} for more details about the experimental setup}
\label{fig:evolve-#2-#4}
\end{figure*}
}



\newcommand{\MakeCodeDynamicsFigures}[3]{%
\begin{figure*}[t]
  \centering

  \begin{subfigure}[t]{0.32\textwidth}
    \centering
    \includegraphics[width=\linewidth]{images/#1-hist-epoch-0-train.png}
    \caption{#2 Epoch 0 – Histogram}\label{fig:hist-0-train-#1}
  \end{subfigure}\hfill
  \begin{subfigure}[t]{0.32\textwidth}
    \centering
    \includegraphics[width=\linewidth]{images/#1-qq-train-epoch-0-train.png}
    \caption{#2 Epoch 0 – QQ Plot}\label{fig:qq-0-train-#1}
  \end{subfigure}\hfill
  \begin{subfigure}[t]{0.32\textwidth}
    \centering
    \includegraphics[width=\linewidth]{images/#1-singular-vals-train-epoch-0-train.png}
    \caption{#2 Epoch 0 – Singular Values}\label{fig:svd-0-train-#1}
  \end{subfigure}

  \par\medskip

  \begin{subfigure}[t]{0.32\textwidth}
    \centering
    \includegraphics[width=\linewidth]{images/#1-hist-epoch-10-train.png}
    \caption{#2 Epoch 50 – Histogram}\label{fig:hist-10-train-#1}
  \end{subfigure}\hfill
  \begin{subfigure}[t]{0.32\textwidth}
    \centering
    \includegraphics[width=\linewidth]{images/#1-qq-train-epoch-10-train.png}
    \caption{#2 Epoch 50 – QQ Plot}\label{fig:qq-10-train-#1}
  \end{subfigure}\hfill
  \begin{subfigure}[t]{0.32\textwidth}
    \centering
    \includegraphics[width=\linewidth]{images/#1-singular-vals-train-epoch-10-train.png}
    \caption{#2 Epoch 50 – Singular Values}\label{fig:svd-10-train-#1}
  \end{subfigure}

  \par\medskip

  \begin{subfigure}[t]{0.32\textwidth}
    \centering
    \includegraphics[width=\linewidth]{images/#1-hist-epoch-25-train.png}
    \caption{#2 Epoch 125 – Histogram}\label{fig:hist-25-train-#1}
  \end{subfigure}\hfill
  \begin{subfigure}[t]{0.32\textwidth}
    \centering
    \includegraphics[width=\linewidth]{images/#1-qq-train-epoch-25-train.png}
    \caption{#2 Epoch 125 – QQ Plot}\label{fig:qq-25-train-#1}
  \end{subfigure}\hfill
  \begin{subfigure}[t]{0.32\textwidth}
    \centering
    \includegraphics[width=\linewidth]{images/#1-singular-vals-train-epoch-25-train.png}
    \caption{#2 Epoch 125 – Singular Values}\label{fig:svd-25-train-#1}
  \end{subfigure}

  \caption{#2: progression of output-code distribution on \emph{train} data across epochs. Each row shows the histogram, QQ plot, and singular-value distribution at epochs 0, 50, and 125. Red lines indicate theoretical densities: standard normal (hist/QQ) and Mar\v{c}enko–Pastur (singular values) with shape #3.}
  \label{fig:dynamics-train-#1}
\end{figure*}

\begin{figure*}[t]
  \centering

  \begin{subfigure}[t]{0.32\textwidth}
    \centering
    \includegraphics[width=\linewidth]{images/#1-hist-epoch-0-test.png}
    \caption{#2 Epoch 0 – Histogram}\label{fig:hist-0-test-#1}
  \end{subfigure}\hfill
  \begin{subfigure}[t]{0.32\textwidth}
    \centering
    \includegraphics[width=\linewidth]{images/#1-qq-test-epoch-0-test.png}
    \caption{#2 Epoch 0 – QQ Plot}\label{fig:qq-0-test-#1}
  \end{subfigure}\hfill
  \begin{subfigure}[t]{0.32\textwidth}
    \centering
    \includegraphics[width=\linewidth]{images/#1-singular-vals-test-epoch-0-test.png}
    \caption{#2 Epoch 0 – Singular Values}\label{fig:svd-0-test-#1}
  \end{subfigure}

  \par\medskip

  \begin{subfigure}[t]{0.32\textwidth}
    \centering
    \includegraphics[width=\linewidth]{images/#1-hist-epoch-10-test.png}
    \caption{#2 Epoch 50 – Histogram}\label{fig:hist-10-test-#1}
  \end{subfigure}\hfill
  \begin{subfigure}[t]{0.32\textwidth}
    \centering
    \includegraphics[width=\linewidth]{images/#1-qq-test-epoch-10-test.png}
    \caption{#2 Epoch 50 – QQ Plot}\label{fig:qq-10-test-#1}
  \end{subfigure}\hfill
  \begin{subfigure}[t]{0.32\textwidth}
    \centering
    \includegraphics[width=\linewidth]{images/#1-singular-vals-test-epoch-10-test.png}
    \caption{#2 Epoch 50 – Singular Values}\label{fig:svd-10-test-#1}
  \end{subfigure}

  \par\medskip

  \begin{subfigure}[t]{0.32\textwidth}
    \centering
    \includegraphics[width=\linewidth]{images/#1-hist-epoch-25-test.png}
    \caption{#2 Epoch 125 – Histogram}\label{fig:hist-25-test-#1}
  \end{subfigure}\hfill
  \begin{subfigure}[t]{0.32\textwidth}
    \centering
    \includegraphics[width=\linewidth]{images/#1-qq-test-epoch-25-test.png}
    \caption{#2 Epoch 125 – QQ Plot}\label{fig:qq-25-test-#1}
  \end{subfigure}\hfill
  \begin{subfigure}[t]{0.32\textwidth}
    \centering
    \includegraphics[width=\linewidth]{images/#1-singular-vals-test-epoch-25-test.png}
    \caption{#2 Epoch 125 – Singular Values}\label{fig:svd-25-test-#1}
  \end{subfigure}

  \caption{#2: progression of output-code distribution on \emph{test} data across epochs. Each row shows the histogram, QQ plot, and singular-value distribution at epochs 0, 50, and 125. Red lines indicate theoretical densities: standard normal (hist/QQ) and Mar\v{c}enko–Pastur (singular values) with shape #3.}
  \label{fig:dynamics-test-#1}
\end{figure*}
} 



\newcommand{\MakeAECodeDynamicsFigures}[4]{%
\begin{figure*}[t]
  \centering

  \begin{subfigure}[t]{0.32\textwidth}
    \centering
    \AEimg{#1}{#2}{hist-epoch-0-train.png}
    \caption{#3 Epoch 0 – Histogram}\label{fig:ae-hist-0-train-#2}
  \end{subfigure}\hfill
  \begin{subfigure}[t]{0.32\textwidth}
    \centering
    \AEimg{#1}{#2}{qq-train-epoch-0-train.png}
    \caption{#3 Epoch 0 – QQ Plot}\label{fig:ae-qq-0-train-#2}
  \end{subfigure}\hfill
  \begin{subfigure}[t]{0.32\textwidth}
    \centering
    \AEimg{#1}{#2}{singular-vals-train-epoch-0-train.png}
    \caption{#3 Epoch 0 – Singular Values}\label{fig:ae-svd-0-train-#2}
  \end{subfigure}

  \par\medskip

  \begin{subfigure}[t]{0.32\textwidth}
    \centering
    \AEimg{#1}{#2}{hist-epoch-2-train.png}
    \caption{#3 Epoch 10 – Histogram}\label{fig:ae-hist-10-train-#2}
  \end{subfigure}\hfill
  \begin{subfigure}[t]{0.32\textwidth}
    \centering
    \AEimg{#1}{#2}{qq-train-epoch-2-train.png}
    \caption{#3 Epoch 10 – QQ Plot}\label{fig:ae-qq-10-train-#2}
  \end{subfigure}\hfill
  \begin{subfigure}[t]{0.32\textwidth}
    \centering
    \AEimg{#1}{#2}{singular-vals-train-epoch-2-train.png}
    \caption{#3 Epoch 10 – Singular Values}\label{fig:ae-svd-10-train-#2}
  \end{subfigure}

  \par\medskip

  \begin{subfigure}[t]{0.32\textwidth}
    \centering
    \AEimg{#1}{#2}{hist-epoch-10-train.png}
    \caption{#3 Epoch 50 – Histogram}\label{fig:ae-hist-25-train-#2}
  \end{subfigure}\hfill
  \begin{subfigure}[t]{0.32\textwidth}
    \centering
    \AEimg{#1}{#2}{qq-train-epoch-10-train.png}
    \caption{#3 Epoch 50 – QQ Plot}\label{fig:ae-qq-25-train-#2}
  \end{subfigure}\hfill
  \begin{subfigure}[t]{0.32\textwidth}
    \centering
    \AEimg{#1}{#2}{singular-vals-train-epoch-10-train.png}
    \caption{#3 Epoch 50 – Singular Values}\label{fig:ae-svd-25-train-#2}
  \end{subfigure}

  \caption{#3 (autoencoder): progression of output-code distribution on \emph{train} data across epochs 0, 10, and 50. Red lines show standard normal (hist/QQ) and Mar\v{c}enko–Pastur (singular values) with shape #4.}
  \label{fig:ae-dynamics-train-#2}
\end{figure*}

\begin{figure*}[t]
  \centering

  \begin{subfigure}[t]{0.32\textwidth}
    \centering
    \AEimg{#1}{#2}{hist-epoch-0-test.png}
    \caption{#3 Epoch 0 – Histogram}\label{fig:ae-hist-0-test-#2}
  \end{subfigure}\hfill
  \begin{subfigure}[t]{0.32\textwidth}
    \centering
    \AEimg{#1}{#2}{qq-test-epoch-0-test.png}
    \caption{#3 Epoch 0 – QQ Plot}\label{fig:ae-qq-0-test-#2}
  \end{subfigure}\hfill
  \begin{subfigure}[t]{0.32\textwidth}
    \centering
    \AEimg{#1}{#2}{singular-vals-test-epoch-0-test.png}
    \caption{#3 Epoch 0 – Singular Values}\label{fig:ae-svd-0-test-#2}
  \end{subfigure}

  \par\medskip

  \begin{subfigure}[t]{0.32\textwidth}
    \centering
    \AEimg{#1}{#2}{hist-epoch-2-test.png}
    \caption{#3 Epoch 10 – Histogram}\label{fig:ae-hist-10-test-#2}
  \end{subfigure}\hfill
  \begin{subfigure}[t]{0.32\textwidth}
    \centering
    \AEimg{#1}{#2}{qq-test-epoch-2-test.png}
    \caption{#3 Epoch 10 – QQ Plot}\label{fig:ae-qq-10-test-#2}
  \end{subfigure}\hfill
  \begin{subfigure}[t]{0.32\textwidth}
    \centering
    \AEimg{#1}{#2}{singular-vals-test-epoch-2-test.png}
    \caption{#3 Epoch 10 – Singular Values}\label{fig:ae-svd-10-test-#2}
  \end{subfigure}

  \par\medskip

  \begin{subfigure}[t]{0.32\textwidth}
    \centering
    \AEimg{#1}{#2}{hist-epoch-10-test.png}
    \caption{#3 Epoch 50 – Histogram}\label{fig:ae-hist-25-test-#2}
  \end{subfigure}\hfill
  \begin{subfigure}[t]{0.32\textwidth}
    \centering
    \AEimg{#1}{#2}{qq-test-epoch-10-test.png}
    \caption{#3 Epoch 50 – QQ Plot}\label{fig:ae-qq-25-test-#2}
  \end{subfigure}\hfill
  \begin{subfigure}[t]{0.32\textwidth}
    \centering
    \AEimg{#1}{#2}{singular-vals-test-epoch-10-test.png}
    \caption{#3 Epoch 50 – Singular Values}\label{fig:ae-svd-25-test-#2}
  \end{subfigure}

  \caption{#3 (autoencoder): progression of output-code distribution on \emph{test} data across epochs 0, 10, and 50. Red lines show standard normal (hist/QQ) and Mar\v{c}enko–Pastur (singular values) with shape #4.}
  \label{fig:ae-dynamics-test-#2}
\end{figure*}
} 

\section{Theory}
\label{app:theory}

The foundational result in this context links the free Poisson distribution to the maximization of a specific functional, which combines the Voiculescu free entropy with a logarithmic potential term from \cite{hiai2006semicircle} is 

\begin{proposition}[Maximization Principle for Free Poisson Distribution] \label{prop:free}
Let $\mu$ be a probability measure on $\mathbb{R}$. The free entropy functional $\Phi_\theta(\mu)$ is defined as:
\[
    \Psi_\theta(\mu) = \chi(\mu) - \int \left(\lambda - (\theta-1)\log(\lambda)\right)d\mu(\lambda) 
\]
where $\theta > 0$ and $\chi(\mu)$ is the Voiculescu free entropy given by:
\[
    \chi(\mu) = \iint \log|\lambda-\tilde{\lambda}| d\mu(\lambda)d\mu(\tilde{\lambda}) 
\]
The unique probability measure that maximizes the functional $\Psi_\theta(\mu)$ is the free Poisson distribution, $\mu^{\textsf{F-P}}_{\theta}$ \cite[Proposition 5.3.7]{hiai2006semicircle}.
\end{proposition}

The functional $\Psi_\theta(\mu)$ can be interpreted as a \emph{free energy}, where $\chi(\mu)$ represents Voiculescu free entropy. The fact that the free Poisson law uniquely maximizes this functional makes it a fundamental object of study, analogous to how the Gaussian distribution maximizes classical entropy for a fixed variance.

\subsection{Maximization Principle for the Mar\v{c}enko-Pastur Distribution}

The maximization principle for the free Poisson distribution can be extended to the Mar\v{c}enko-Pastur distribution. This extension is achieved through a direct scaling relationship.

\begin{proposition}[Maximization Principle for Mar\v{c}enko-Pastur Distribution]
Let $\nu$ be a probability measure on $\mathbb{R}$ and let $c > 0$ be a parameter. Define the functional $\Phi_c(\nu)$ as:
$$ \Phi_c(\nu) = \chi(\nu) - \int \left(\frac{\lambda}{c} - \left(\frac{1}{c}-1\right)\log(\lambda)\right)d\nu(\lambda) $$
The unique probability measure that maximizes the functional $\Psi_c(\nu)$ is the Mar\v{c}enko-Pastur distribution, $\mu_{c}^{\textsf{M-P}}$.
\end{proposition}

\begin{proof}
 Let $\mu$ be any measure, define $\nu$ as the pushforward of $\mu$, under that map $T(\xi) = \xi/\theta =: \lambda$. Then we see that for any measure $A$, we have that 
\[
    \nu(A) = \mu\left(T^{-1}(A)\right) = \mu\left(\theta A\right)
\]
Note that $\xi = \lambda \theta$, and 
\begin{align*}
    \chi(\mu) &= \iint \log|\xi - \tilde{\xi}|\,d\mu(\xi)d\mu(\tilde{\xi}) \\
    &= \iint \log|\lambda \theta - \tilde{\lambda}\theta |\,d\nu(\lambda)d\nu(\tilde{\lambda}) \\
    &= \underbrace{\iint \log|\lambda - \tilde{\lambda}|\,d\nu(\lambda)d\nu(\tilde{\lambda})}_{\chi(\nu)} + \log(\theta)
\end{align*}
and
\begin{align*}
    \int \left(\xi - (\theta-1)\log(\xi)\right)d\mu(\xi) &= \int \left(\lambda \theta - (\theta-1)\log(\lambda \theta)\right)d\nu(\lambda) \\
    &= \int \left(\lambda\theta - (\theta-1)\log(\lambda)\right)d\nu(\lambda) \\
    &\quad - (\theta-1)\log(\theta)
\end{align*}
Adding the two and recalling that $c = \frac{1}{\theta}$ we see that
\begin{align*}
    \Phi_c(\nu) &= \chi(\nu) - \int \left(\frac{\lambda}{c} - \left(\frac{1}{c}-1\right)\log(\lambda)\right)d\nu(\lambda) \\
    &= \chi(\nu) - \int \left(\lambda\theta - (\theta-1)\log(\lambda)\right)d\nu(\lambda) \\
    &= \chi(\mu)  - \log(\theta)\\
    & \quad - \left[\int \left(\xi - (\theta-1)\log(\xi)\right)d\mu(\xi) + (\theta - 1)\log(\theta)\right]\\
    &= \Psi_\theta(\mu) - \theta \log(\theta)
\end{align*}
For the purpose of maximization with respect to the measure $\nu$, the constant term $\theta\log(\theta)$ can be disregarded. 

From Proposition~\ref{prop:free}, we now that $\mu_\theta^{\textsf{F-P}}$ is the unique maximizer of $\Psi_\theta$. Thus, the pushforward $\mu_\theta^{\textsf{F-P}}$ under the map $T(\xi) = \xi/\theta$ is the unique maximizer of $\Phi_c(\nu)$. Since this pushforward is exactly $\mu_c^{\textsf{M-P}}$, we get the needed result. 
\end{proof}

Note that this loss can be easily implemented as follows in PyTorch. 

{\centering
\begin{lstlisting}[style=py, caption={PyTorch code for computing the matricial free energy and the free loss as in (\ref{eq:freeloss-normal}) and (\ref{eq:free loss}), respectively.}, label={lst:free-energy}]
def g(svals_sq, c):  
    d = svals_sq.shape[0]
    return (svals_sq.sum()/d - (1/c - 1) * torch.log(svals_sq)).sum() / d

def matricial_free_energy(Y):
   if isinstance(Y, torch.Tensor):
       # Y is a PyTorch tensor stored as batch_size x features matrix when obtained through a PyTorch model
       b, d = Y.shape[0], Y.shape[1] 
   else: # if a matrix in math form (i.e. not a PyTorch tensor) as a features x batch_size matrix
        d, b = Y.shape[0], Y.shape[1] 
        Y = torch.tensor(Y)
   assert d < b, f"Expected d < b, but got d={d} and b={b}"
   svals_sq = torch.svd(Y) ** 2
    
   # Note:
   # pdist returns a 1D tensor containing the upper triangular part (excluding the diagonal) of the distance matrix. 
   # So for an N dimensional vectors, you get N*(N-1)/2 distances  - to compute chiX we need to double it

   chi_Y =  2 * torch.log(torch.pdist(svals_sq.view(d, 1), p = 1)).sum() / (d * (d - 1))
   Phi_c_Y = chi_Y - g(svals_sq, d/b)
   return Phi_c_Y

def free_loss(Y):
     return -matricial_free_energy(Y)
\end{lstlisting}
}

\FloatBarrier

\section{Optimization Dynamics}
\label{app:dynamics}

Figures~\ref{fig:dynamics-train} and \ref{fig:dynamics-test} show how the histogram of the entries, the qq-plot of the entries and the singular value distribution for the code matrix for a single batch of data evolves during training. Here we see that the even though the Free Loss was close to the theoretical minimum after 100 epochs, at that point, the code matrix is not Gaussian. This only occurs, mucn later in training. 

Finally, Figure~\ref{fig:dynamics-joy} shows the joyplot for evolution of the histogram of the entries.

\begin{figure*}[t]
  \centering

  \begin{subfigure}[t]{0.32\textwidth}
    \centering
    \includegraphics[width=\linewidth]{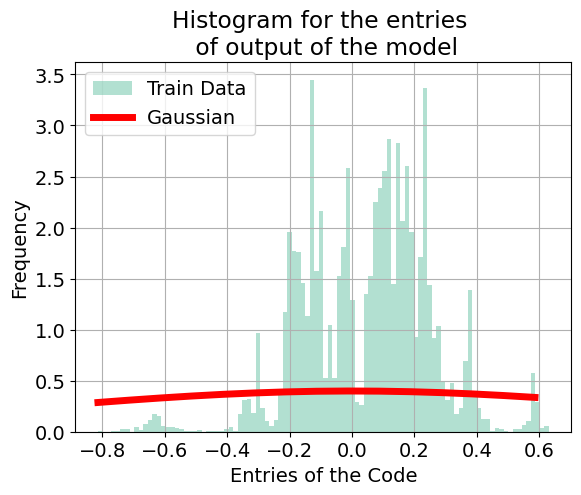}
    \caption{Epoch 0 – Empirical pdf of the entries vs pdf of $\mathcal{N}(0,1)$}\label{fig:hist-0-train}
  \end{subfigure}\hfill
  \begin{subfigure}[t]{0.32\textwidth}
    \centering
    \includegraphics[width=\linewidth]{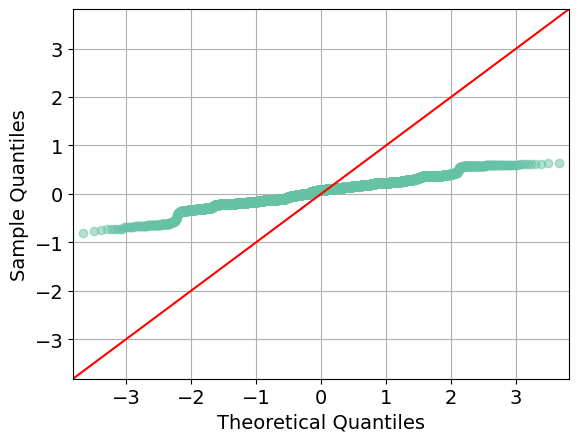}
    \caption{Epoch 0 – Q-Q plot relative to a standard $\mathcal{N}(0,1)$ distribution}\label{fig:qq-0-train}
  \end{subfigure}\hfill
  \begin{subfigure}[t]{0.32\textwidth}
    \centering
    \includegraphics[width=\linewidth]{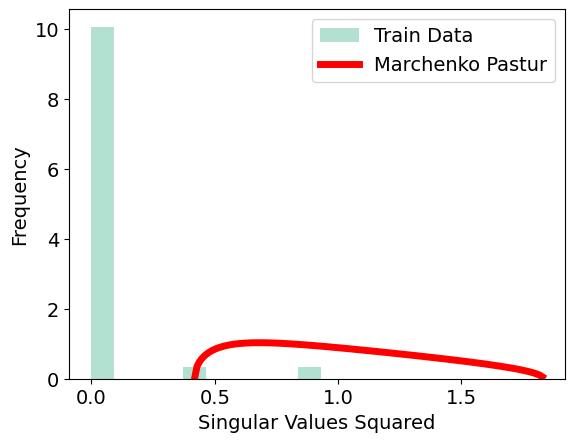}
    \caption{Epoch 0 – Distribution of sample covariance matrix eigenvalues}\label{fig:svd-0-train}
  \end{subfigure}

  \par\medskip

  \begin{subfigure}[t]{0.32\textwidth}
    \centering
    \includegraphics[width=\linewidth]{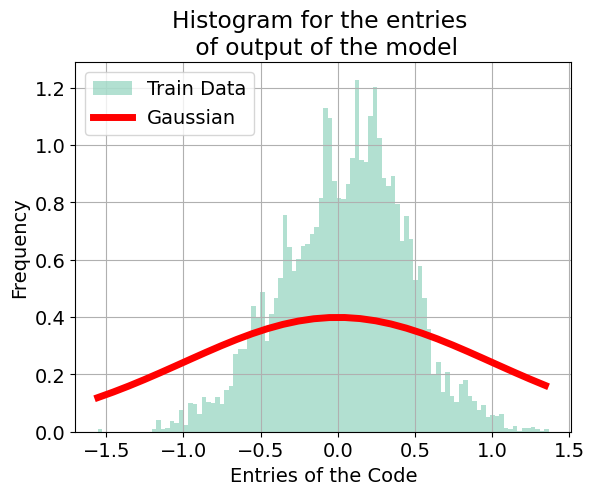}
    \caption{Epoch 10 – Empirical pdf of the entries vs pdf of $\mathcal{N}(0,1)$}\label{fig:hist-10-train}
  \end{subfigure}\hfill
  \begin{subfigure}[t]{0.32\textwidth}
    \centering
    \includegraphics[width=\linewidth]{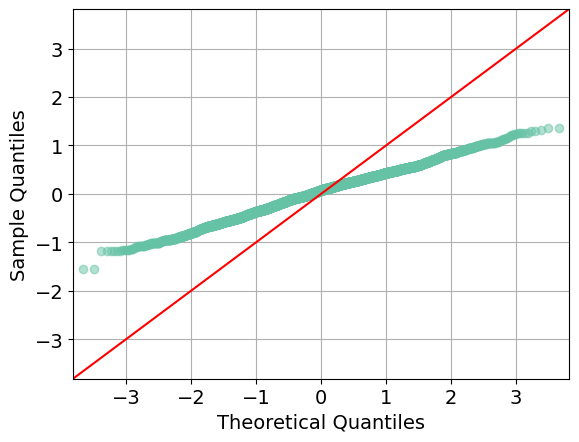}
    \caption{Epoch 10 – Q-Q plot relative to a standard $\mathcal{N}(0,1)$ distribution}\label{fig:qq-10-train}
  \end{subfigure}\hfill
  \begin{subfigure}[t]{0.32\textwidth}
    \centering
    \includegraphics[width=\linewidth]{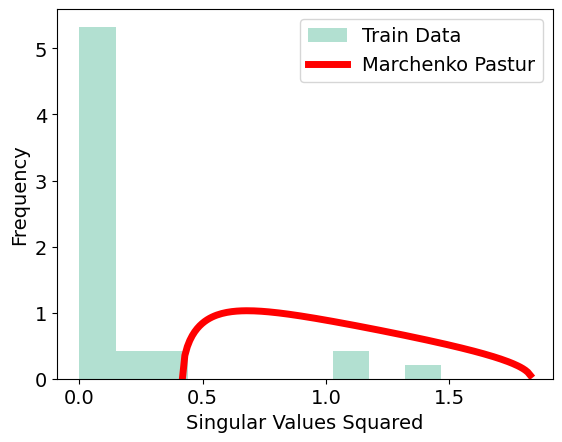}
    \caption{Epoch 10 – Distribution of sample covariance matrix eigenvalues}\label{fig:svd-10-train}
  \end{subfigure}

  \par\medskip

  \begin{subfigure}[t]{0.32\textwidth}
    \centering
    \includegraphics[width=\linewidth]{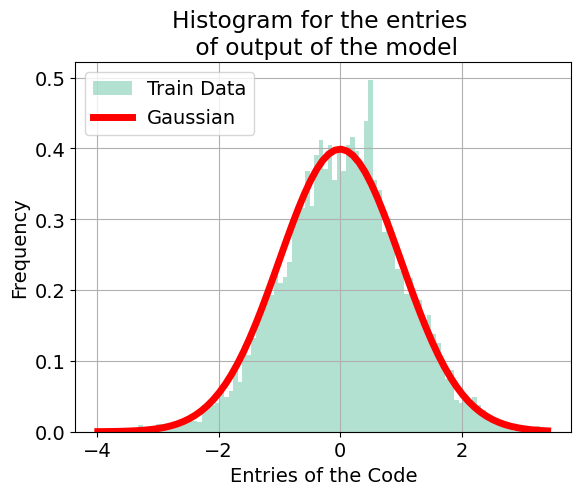}
    \caption{Epoch 100 – Empirical pdf of the entries vs pdf of $\mathcal{N}(0,1)$}\label{fig:hist-100-train}
  \end{subfigure}\hfill
  \begin{subfigure}[t]{0.32\textwidth}
    \centering
    \includegraphics[width=\linewidth]{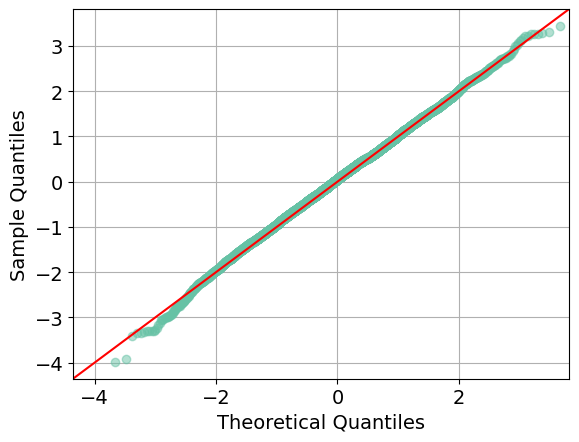}
    \caption{Epoch 100 – Q-Q plot relative to a standard $\mathcal{N}(0,1)$ distribution}\label{fig:qq-100-train}
  \end{subfigure}\hfill
  \begin{subfigure}[t]{0.32\textwidth}
    \centering
    \includegraphics[width=\linewidth]{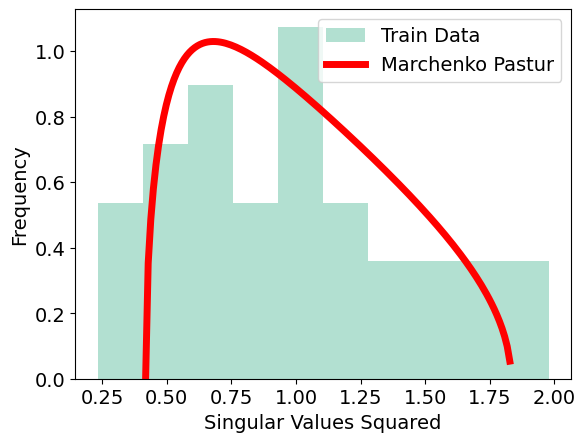}
    \caption{Epoch 100 – Distribution of sample covariance matrix eigenvalues}\label{fig:svd-100-train}
  \end{subfigure}

  \par\medskip

  \begin{subfigure}[t]{0.32\textwidth}
    \centering
    \includegraphics[width=\linewidth]{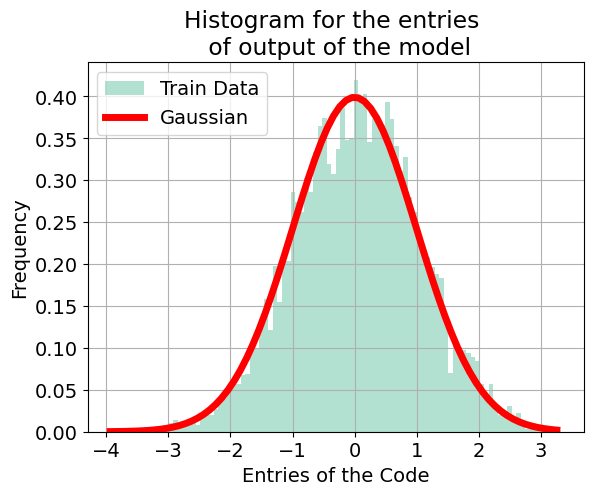}
    \caption{Epoch 1000 – Empirical pdf of the entries vs pdf of $\mathcal{N}(0,1)$}\label{fig:hist-1000-train}
  \end{subfigure}\hfill
  \begin{subfigure}[t]{0.32\textwidth}
    \centering
    \includegraphics[width=\linewidth]{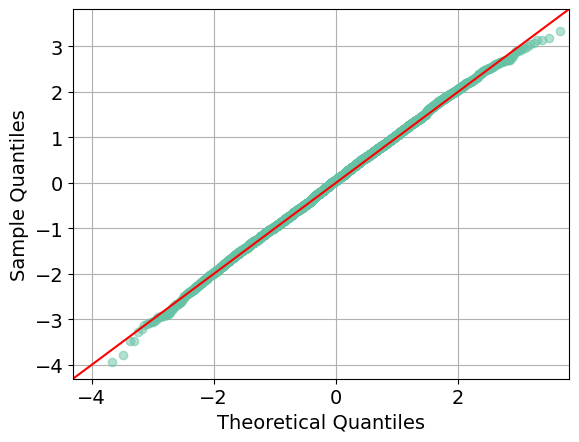}
    \caption{Epoch 1000 – Q-Q plot relative to a standard $\mathcal{N}(0,1)$ distribution}\label{fig:qq-1000-train}
  \end{subfigure}\hfill
  \begin{subfigure}[t]{0.32\textwidth}
    \centering
    \includegraphics[width=\linewidth]{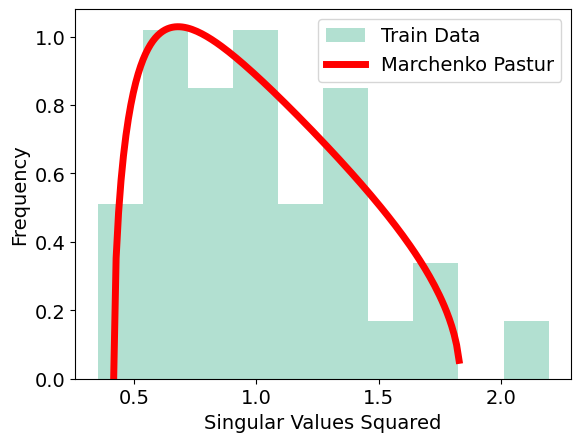}
    \caption{Epoch 1000 – Distribution of sample covariance matrix eigenvalues}\label{fig:svd-1000-train}
  \end{subfigure}

  \caption{Progression of output code distribution during training. Each row shows the histogram, QQ plot, and singular value distribution for different epochs: (a–c) initialization, (d–f) epoch 10, (g–i) epoch 100, and (j–l) epoch 1000. Red lines indicate theoretical densities: standard normal for histogram/QQ and Marchenko–Pastur for singular values with shape $c=32/256$.}
  \label{fig:dynamics-train}
\end{figure*}

\begin{figure*}[t]
  \centering

  \begin{subfigure}[t]{0.32\textwidth}
    \centering
    \includegraphics[width=\linewidth]{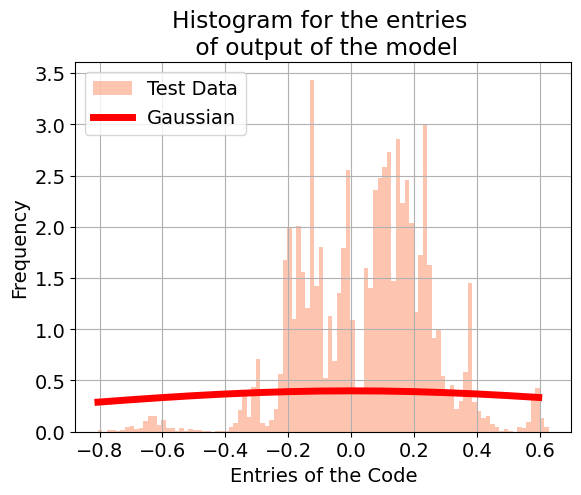}
    \caption{Epoch 0 – Empirical pdf of the entries vs pdf of $\mathcal{N}(0,1)$}\label{fig:hist-0-test}
  \end{subfigure}\hfill
  \begin{subfigure}[t]{0.32\textwidth}
    \centering
    \includegraphics[width=\linewidth]{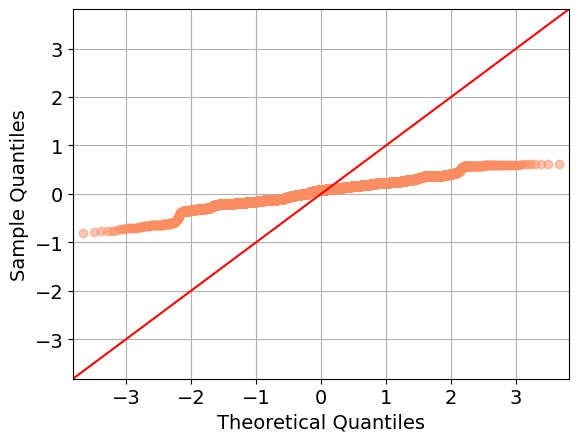}
    \caption{Epoch 0 – Q-Q plot relative to a standard $\mathcal{N}(0,1)$ distribution}\label{fig:qq-0-test}
  \end{subfigure}\hfill
  \begin{subfigure}[t]{0.32\textwidth}
    \centering
    \includegraphics[width=\linewidth]{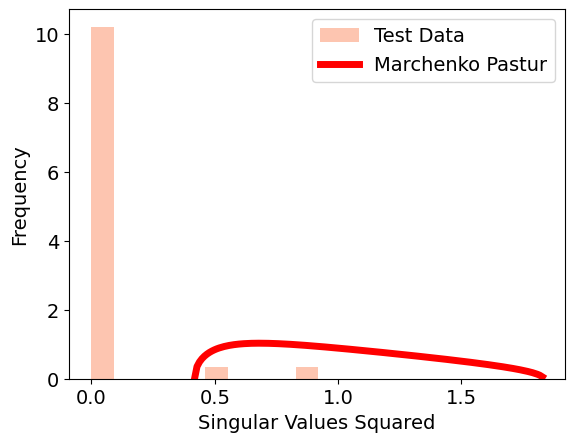}
    \caption{Epoch 0 – Distribution of sample covariance matrix eigenvalues}\label{fig:svd-0-test}
  \end{subfigure}

  \par\medskip

  \begin{subfigure}[t]{0.32\textwidth}
    \centering
    \includegraphics[width=\linewidth]{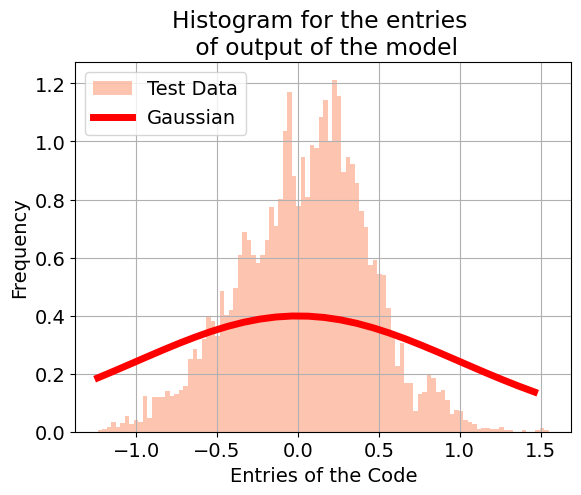}
    \caption{Epoch 10 – Empirical pdf of the entries vs pdf of $\mathcal{N}(0,1)$}\label{fig:hist-10-test}
  \end{subfigure}\hfill
  \begin{subfigure}[t]{0.32\textwidth}
    \centering
    \includegraphics[width=\linewidth]{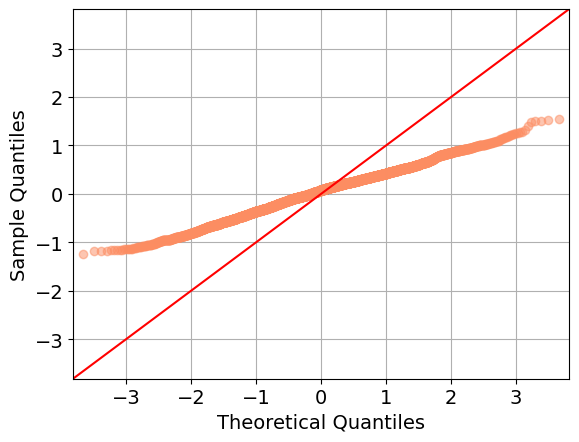}
    \caption{Epoch 10 – Q-Q plot relative to a standard $\mathcal{N}(0,1)$ distribution}\label{fig:qq-10-test}
  \end{subfigure}\hfill
  \begin{subfigure}[t]{0.32\textwidth}
    \centering
    \includegraphics[width=\linewidth]{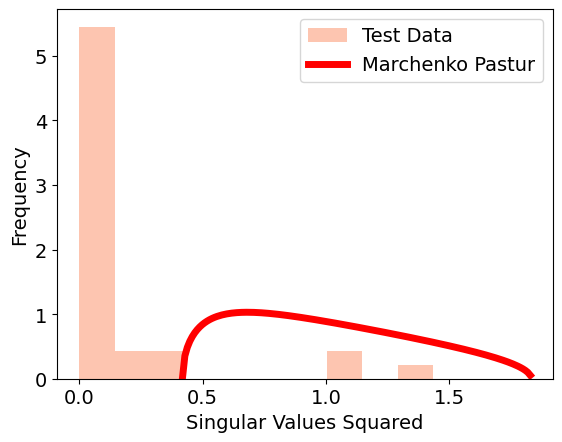}
    \caption{Epoch 10 – Distribution of sample covariance matrix eigenvalues}\label{fig:svd-10-test}
  \end{subfigure}

  \par\medskip

  \begin{subfigure}[t]{0.32\textwidth}
    \centering
    \includegraphics[width=\linewidth]{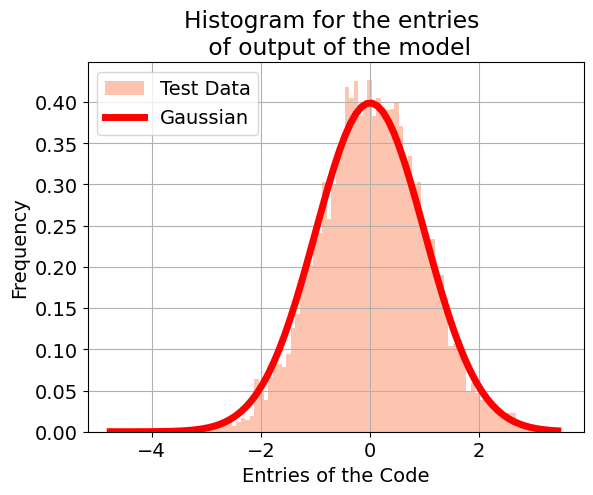}
    \caption{Epoch 100 – Empirical pdf of the entries vs pdf of $\mathcal{N}(0,1)$}\label{fig:hist-100-test}
  \end{subfigure}\hfill
  \begin{subfigure}[t]{0.32\textwidth}
    \centering
    \includegraphics[width=\linewidth]{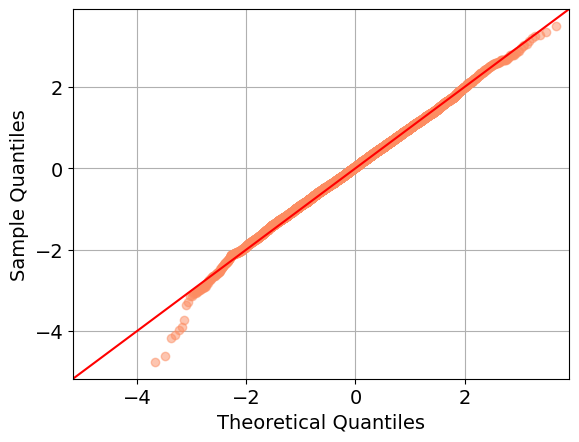}
    \caption{Epoch 100 – Q-Q plot relative to a standard $\mathcal{N}(0,1)$ distribution}\label{fig:qq-100-test}
  \end{subfigure}\hfill
  \begin{subfigure}[t]{0.32\textwidth}
    \centering
    \includegraphics[width=\linewidth]{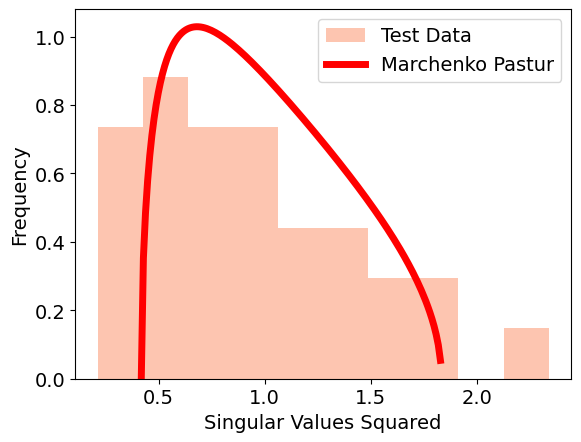}
    \caption{Epoch 100 – Distribution of sample covariance matrix eigenvalues}\label{fig:svd-100-test}
  \end{subfigure}

  \par\medskip

  \begin{subfigure}[t]{0.32\textwidth}
    \centering
    \includegraphics[width=\linewidth]{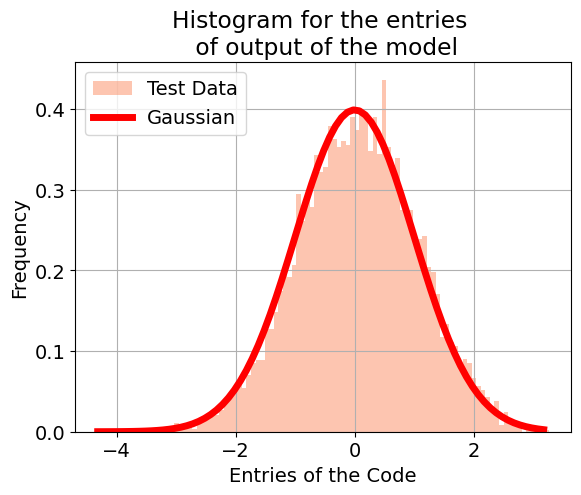}
    \caption{Epoch 1000 – Empirical pdf of the entries vs pdf of $\mathcal{N}(0,1)$}\label{fig:hist-1000-test}
  \end{subfigure}\hfill
  \begin{subfigure}[t]{0.32\textwidth}
    \centering
    \includegraphics[width=\linewidth]{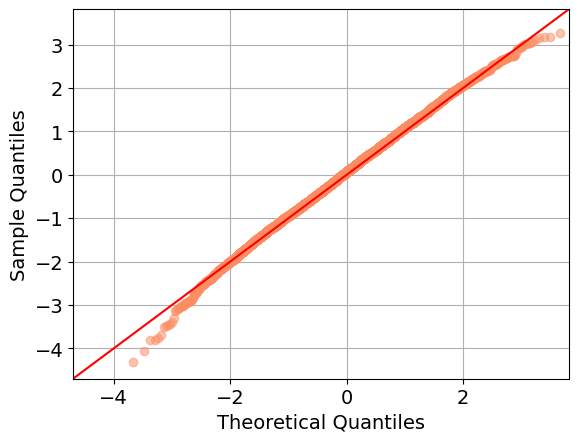}
    \caption{Epoch 1000 – Q-Q plot relative to a standard $\mathcal{N}(0,1)$ distribution}\label{fig:qq-1000-test}
  \end{subfigure}\hfill
  \begin{subfigure}[t]{0.32\textwidth}
    \centering
    \includegraphics[width=\linewidth]{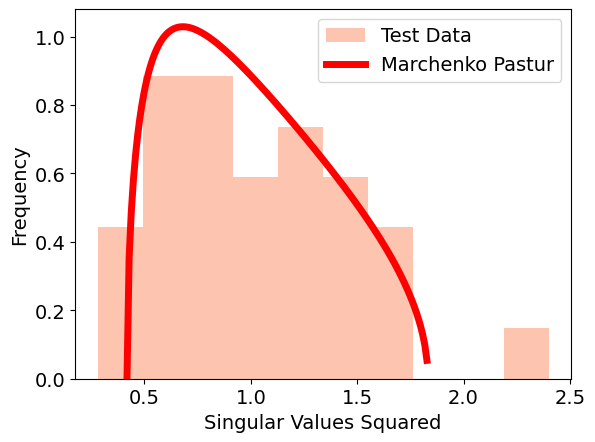}
    \caption{Epoch 1000 – Distribution of sample covariance matrix eigenvalues}\label{fig:svd-1000-test}
  \end{subfigure}

  \caption{Progression of output code distribution for test data across epochs. Each row shows the histogram, QQ plot, and singular value distribution: (a–c) initialization, (d–f) epoch 10, (g–i) epoch 100, and (j–l) epoch 1000. Red lines indicate theoretical densities: standard normal (hist/QQ) and Mar\v{c}enko–Pastur (singular values) with shape $c=32/256$.}
  \label{fig:dynamics-test}
\end{figure*}

\FloatBarrier

\begin{figure*}[!ht]
  \centering
  \begin{subfigure}[t]{0.49\textwidth}
    \centering
    \includegraphics[width=\linewidth]{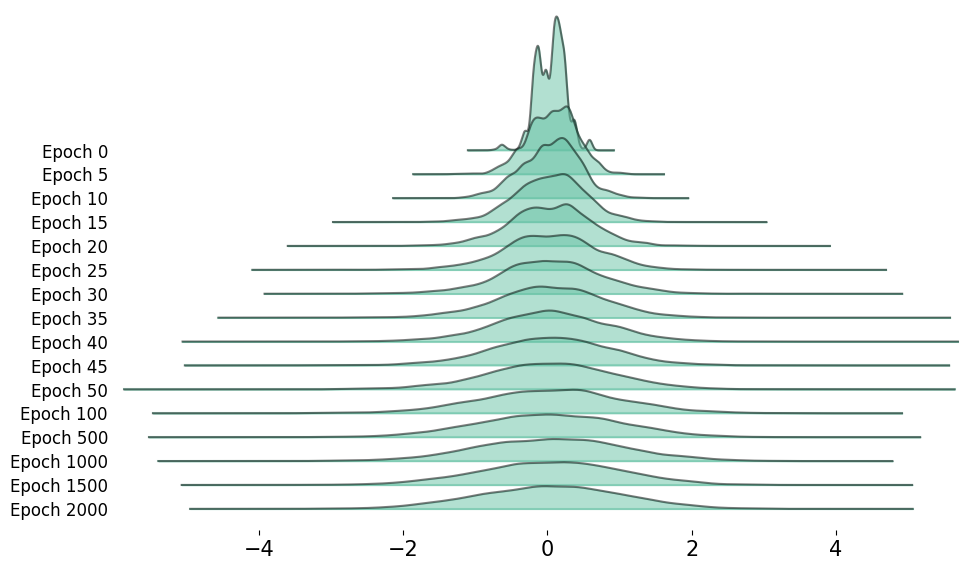}
    \caption{Train Histogram}\label{fig:joy-train}
  \end{subfigure}\hfill
  \begin{subfigure}[t]{0.49\textwidth}
    \centering
    \includegraphics[width=\linewidth]{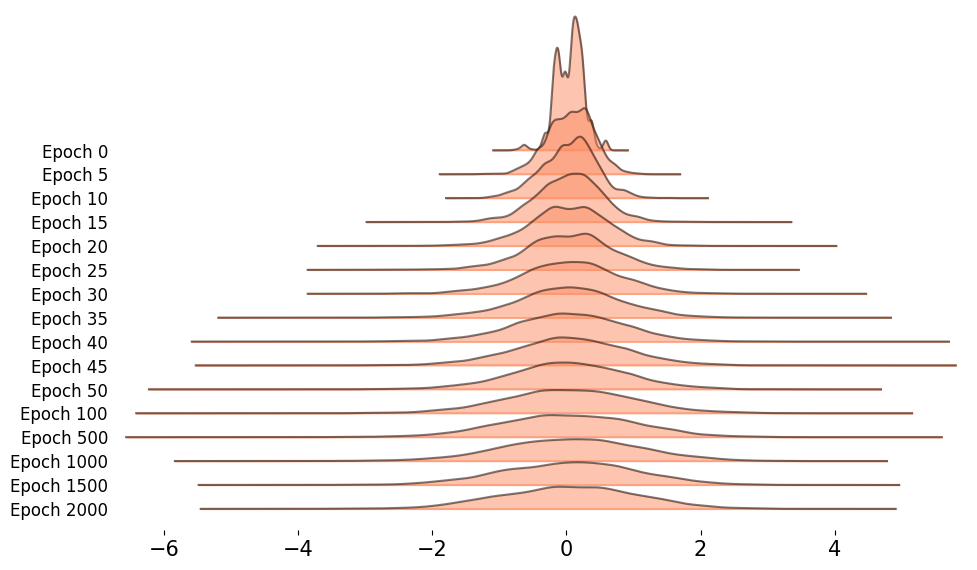}
    \caption{Test Histogram}\label{fig:joy-test}
  \end{subfigure}
  \caption{Histogram for entries of a code matrix over training. Initially non-Gaussian, the shape becomes more Gaussian with training; variance is initially too large and then contracts after $\sim$50 epochs toward the correct moments.}
  \label{fig:dynamics-joy}
\end{figure*}

\section{Batch Size and Dimension}
\label{app:batchsize}

We evaluate Gaussianization across batch sizes $b\!\in\!\{64,128,256,512\}$ and latent dimensions $d\!\in\!\{2,4,8,16,32\}$ with $5$ independent trials per setting.
For each trained model we report:
(i) the Kolmogorov–Smirnov statistic \textbf{KS} on the flattened codes;
(ii) the relative excess OT cost $\Delta_{\texttt{OT}}$ in \eqref{eq:ot}; and
(iii) the relative deviation of the matricial Free Loss
$\mathrm{RelErr}_{\text{free}}
=\big|\big(\mathcal{L}_{\textrm{free}}(Z)-\mathcal{L}_{\textrm{free}}(G)\big)\big/\mathcal{L}_{\textrm{free}}(G)\big|,$
with $Z\in\mathbb{R}^{d\times b}$ and $G$ i.i.d.\ Gaussian of matching shape. We plot the quantities versus the batch size in Figure~\ref{fig:batch-panels-app} and versus dimension in Figure~\ref{fig:dim-panels-app}.
All error bars are \emph{±\,Standard Mean Error} across the 5 trials. 

\begin{figure*}[t]
  \centering
  \begin{subfigure}[t]{0.32\linewidth}
    \centering
    \includegraphics[width=\linewidth]{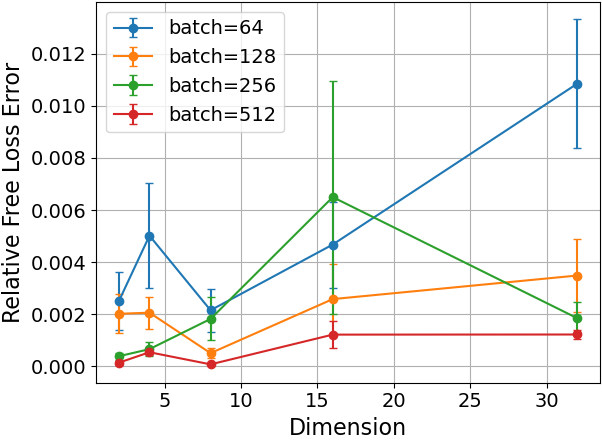}
    \caption{\textbf{$\mathrm{RelErr}_{\text{free}}$ vs.\ $d$.}
    Values are small across all $b$.}
    \label{fig:free-vs-dim-app}
  \end{subfigure}\hfill
  \begin{subfigure}[t]{0.32\linewidth}
    \centering
    \includegraphics[width=\linewidth]{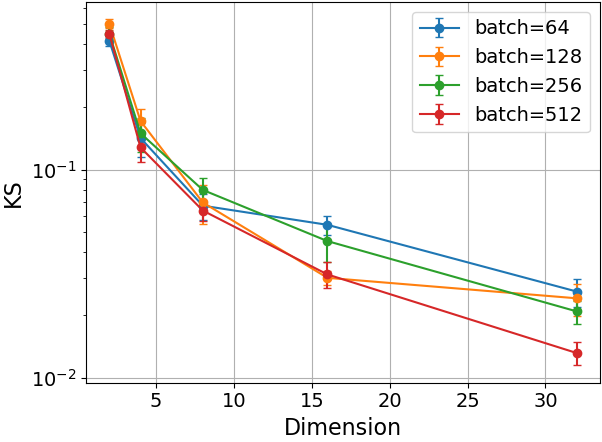}
    \caption{\textbf{KS vs.\ $d$.}
    KS drops sharply as dimension increases.}
    \label{fig:ks-vs-dim-app}
  \end{subfigure}\hfill
  \begin{subfigure}[t]{0.32\linewidth}
    \centering
    \includegraphics[width=\linewidth]{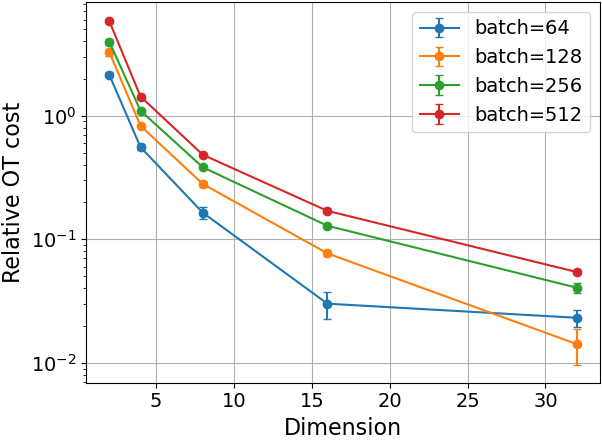}
    \caption{\textbf{$\Delta_{\texttt{OT}}$ vs.\ $d$ (log $y$).} Relative OT cost drops sharply as dimension increases.}
    \label{fig:ot-vs-dim-app}
  \end{subfigure}
  \caption{\textbf{Scaling with latent dimension $d$.}
  Each curve fixes a batch size $b$ and varies $d$.
  Error bars show \emph{±\,SEM} across 5 trials. See Section~\ref{app:batchsize} for more details.}
  \label{fig:dim-panels-app}
\end{figure*}

\begin{figure*}[t]
  \centering
  \begin{subfigure}[t]{0.32\linewidth}
    \centering
    \includegraphics[width=\linewidth]{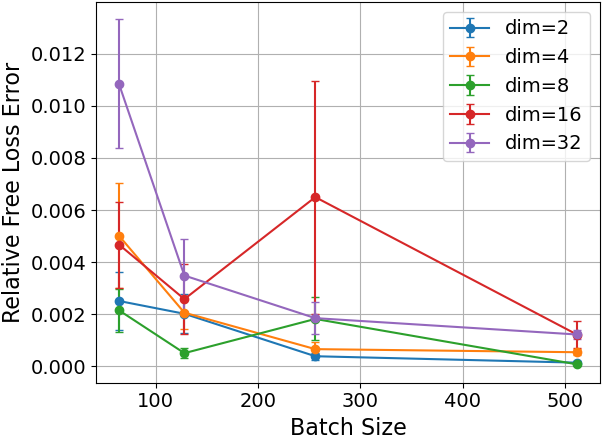}
    \caption{\textbf{$\mathrm{RelErr}_{\text{free}}$ vs.\ $b$.}
    Means remain small for all batch sizes.}
    \label{fig:free-vs-batch-app}
  \end{subfigure}\hfill
  \begin{subfigure}[t]{0.32\linewidth}
    \centering
    \includegraphics[width=\linewidth]{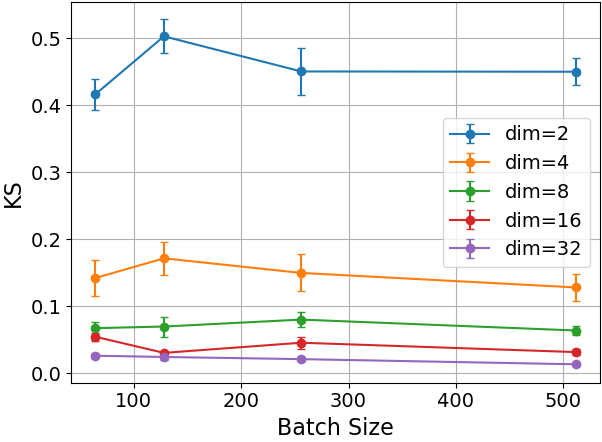}
    \caption{\textbf{KS vs.\ $b$.} KS stays flat, as a function of batch size. }
    \label{fig:ks-vs-batch-app}
  \end{subfigure}\hfill
  \begin{subfigure}[t]{0.32\linewidth}
    \centering
    \includegraphics[width=\linewidth]{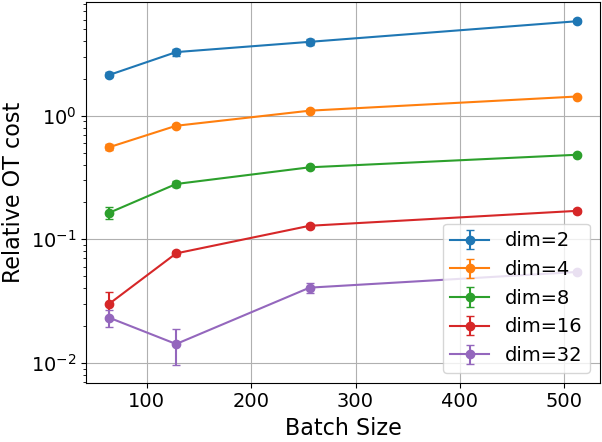}
    \caption{\textbf{$\Delta_{\texttt{OT}}$ vs.\ $b$ (log $y$).}
    $\Delta_{\texttt{OT}}$ tends to increase with $b$.}
    \label{fig:ot-vs-batch-app}
  \end{subfigure}
  \caption{\textbf{Scaling with batch size $b$.}
  Each curve fixes a dimension $d$ and varies $b$.
  Error bars show \emph{±\,SEM} across 5 trials. See Section~\ref{app:batchsize} for more details.}
  \label{fig:batch-panels-app}
\end{figure*}

\paragraph{Takeaways.}
\begin{enumerate}[nosep, itemsep = 2pt]
    \item Both KS and $\Delta_{\texttt{OT}}$ improve rapidly with $d$.
    \item Batch size $b$ has a secondary but visible effect on $\Delta_{\texttt{OT}}$ at low $d$.
    \item $\mathrm{RelErr}_{\text{free}}$ is small throughout.
\end{enumerate}

\section{Real Data}
\label{app:real}

We describe the encoders used to produce \(d\)-dimensional codes that are trained using Free Loss \eqref{eq:freeloss-normal}. We train with Adam (lr = \(10^{-3}\), \(\beta{=}(0.9,0.999)\)), and standard data shuffling each epoch. 

To highlight the effectiveness of using Free Loss, we note that no extra engineering went into the design of these networks. They were generically chosen as reasonable models recent models for each data modality. 

\subsection{Audio Encoders }
\textbf{Front-end.} We form 128-bin log-Mel spectrograms. The resulting input data tensors are shaped as \texttt{(batch, freq=128, time)}.

\textbf{Front-end.} We form 128-bin log-Mel spectrograms. The resulting tensors are shaped as \texttt{(batch, freq=128, time)}.

\textbf{Backbone.} A compact Conformer \citep{gulati2020conformer} encoder with the following parameters: \texttt{d\_model = 64}, \texttt{num\_blocks = 2}, \texttt{nhead = 1}, \texttt{dim\_feedforward = 64}, depthwise convolution \texttt{kernel\_size = 31} (odd; same-padding), \texttt{dropout = 0}.

Each Conformer block follows this structure: $\tfrac{1}{2}$ Feed-Forward Network (FFN) $\to$ Multi-Head Self-Attention (MHSA) $\to$ Conv-module $\to$ $\tfrac{1}{2}$FFN, with residuals and a final LayerNorm. In equation form:
\begin{equation*}
\begin{aligned}
& x \leftarrow x + \tfrac{1}{2}\mathrm{FFN}(x); & x \leftarrow x + \mathrm{MHSA}(x); \\
& x \leftarrow x + \mathrm{Conv}(x); & x \leftarrow x + \tfrac{1}{2}\mathrm{FFN}(x); \\
& x \leftarrow \mathrm{LN}(x).
\end{aligned}
\end{equation*}
The FFN uses GLU gating and SiLU activations. The Conv-module is: Pointwise-conv $\to$ Depthwise 1D conv (groups = channels) $\to$ Batch-Norm $\to$ SiLU $\to$ Pointwise-conv. This is applied on tensors shaped \texttt{(batch, time, dim)}, with necessary permutations between time and channel dimensions.

\textbf{Pooling.} Attention pooling over the time dimension produces a single vector shaped \texttt{(batch, 64)}.

\textbf{Head.} A linear layer from 64 to 32 dimensions to obtain the embedding $z \in \mathbb{R}^{32}$.

\textbf{Batches.} Batch size $b = 64$.

\subsection{Text: Transformer Encoder}

We use the encoder transformer \citep{vaswani2017attention}.

\textbf{Tokens.} Vocabulary size $|V| = \texttt{vocab\_size}$ from the data loader. Sequences are padded or truncated to a fixed maximum length.

\textbf{Backbone.} PyTorch \texttt{nn.TransformerEncoder} with the following parameters: \texttt{d\_model = 128}, \texttt{nhead = 2}, \texttt{num\_encoder\_layers = 3}, \texttt{dim\_feedforward = 128}, \texttt{dropout = 0}, \texttt{batch\_first = True}.

Each encoder layer consists of Multi-Head Self-Attention (MHSA) followed by a Feed-Forward Network (FFN), with residuals and layer normalization. Token embeddings are 128-dimensional and scaled by \( \sqrt{d_{\text{model}}} \). We add sine–cosine positional encodings.

\textbf{Pooling.} A learned attention pooling over the sequence returns a single vector shaped \texttt{(batch, 128)}.

\textbf{Head.} A linear layer from 128 to 32 dimensions to obtain the embedding \( z \in \mathbb{R}^{32} \).

\textbf{Batches.} Batch size \( b = 128 \).

\subsection{Vision: EfficientViT-M2.}

We use EfficientViT from \citep{liu2023efficientvit}.

\textbf{Input.} RGB images resized to \( 224 \times 224 \); per-channel normalization.

\textbf{Backbone.} \texttt{vit.get\_embedder("efficient")} wraps a TIMM EfficientViT-M2 backbone with \texttt{pretrained=False}, drops the classifier (\texttt{num\_classes=0}), and adds a linear projector yielding a 100-dimensional embedding. A linear layer from 100 to 32 dimensions to obtain the embedding \( z \in \mathbb{R}^{32} \).

\textbf{Batches.} Batch size \( b = 128 \).

\subsection{Plots}
\label{sec:real-plots}

We the equivalent of Figure~\ref{fig:2} for an example dataset from each modality. MNIST for image, GTZAN for audio, and IMDB for text. 

\FloatBarrier
  
\MakeEncoderEvolveFig{MNIST}{Encoder (MNIST)}{25}{$c=32/128$}
\MakeEncoderEvolveFig{IMDB}{Encoder (IMDB)}{25}{$c=32/128$}
\MakeEncoderEvolveFig{GTZAN}{Encoder (GTZAN)}{25}{$c=32/64$}




\FloatBarrier

\begin{figure*}[!h]
  \centering
  \begin{subfigure}[t]{0.24\textwidth}
    \centering
    \includegraphics[width=\linewidth]{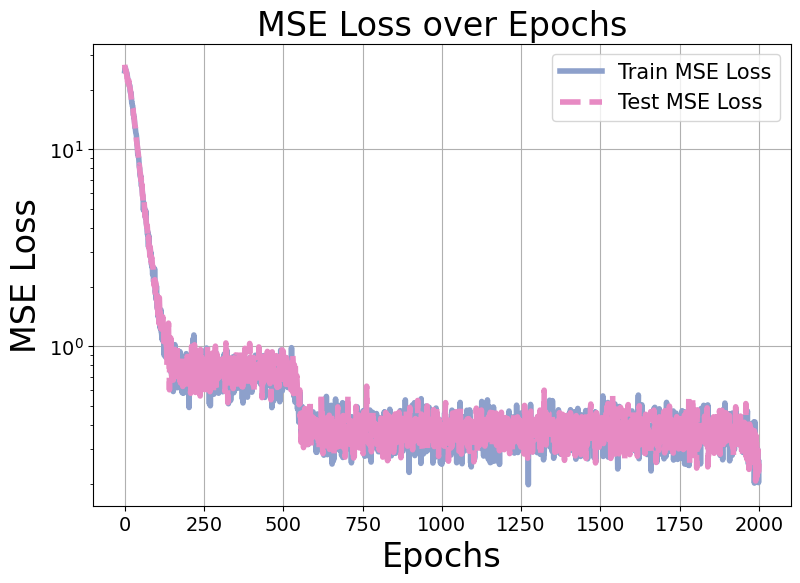}
    \caption{$\tau = 0$ MSE}
    \label{fig:mse-tau0}
  \end{subfigure}\hfill
  \begin{subfigure}[t]{0.24\textwidth}
    \centering
    \includegraphics[width=\linewidth]{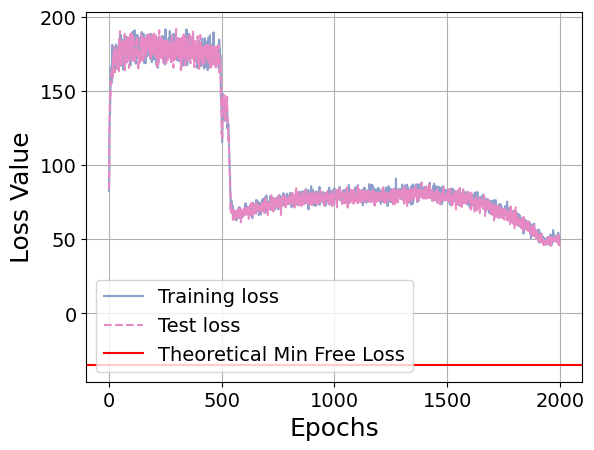}
    \caption{$\tau = 0$ Free Loss}
    \label{fig:free-tau0}
  \end{subfigure}\hfill
  \begin{subfigure}[t]{0.24\textwidth}
    \centering
    \includegraphics[width=\linewidth]{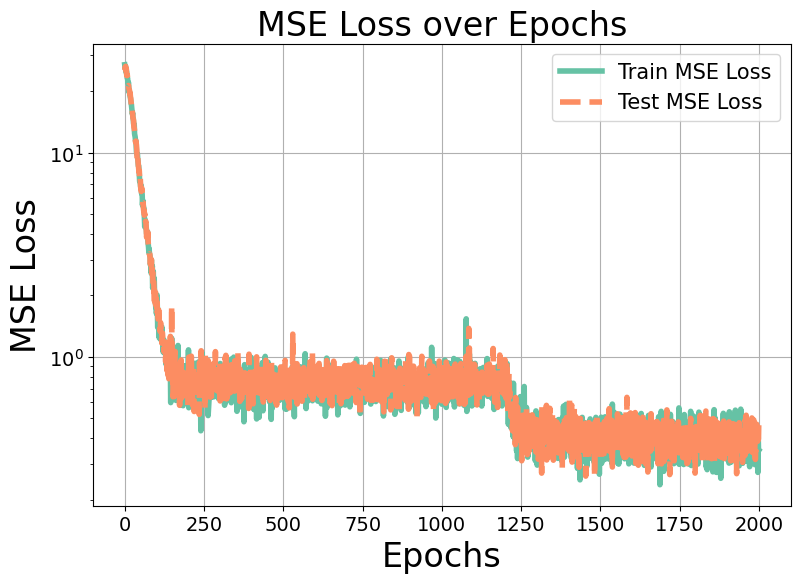}
    \caption{$\tau = 1$ MSE}
    \label{fig:mse-tau1}
  \end{subfigure}\hfill
  \begin{subfigure}[t]{0.24\textwidth}
    \centering
    \includegraphics[width=\linewidth]{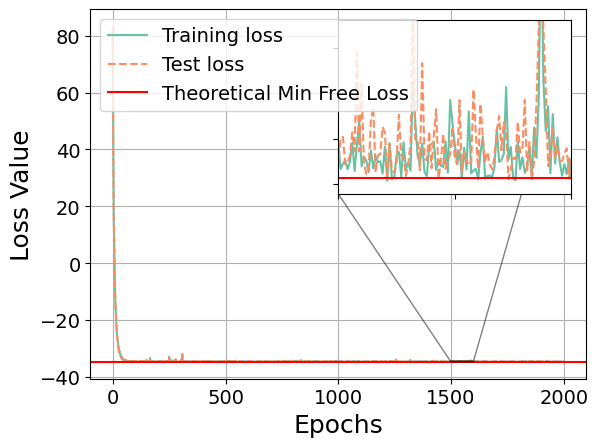}
    \caption{$\tau = 1$ Free Loss}
    \label{fig:free-tau1}
  \end{subfigure}
  \caption{MSE and Free Loss for training an autoencoder with loss $\mathcal{L}_{\text{free}} + \tau\cdot\text{MSE}$, for $\tau=0$ and $\tau=1$.}
  \label{fig:autoencoder-loss}
\end{figure*}

\section{Autoencoder}
\label{app:autoencoder}

\subsection{Chi Squared Data Mixture}
\label{sec:E1}

We begin by plotting the training error curves for the Free Loss regularized autoencoder and the unregularized autoencoder. These can be see in Figure~\ref{fig:autoencoder-loss}. As we can see from the figure. The Free Loss regularized autoencoder successfully minimizes the MSE and Free Loss. While the unregularized autoencoder, minimizes the MSE, but not the Free Loss. 

Next we explore the Gaussianity metrics, for the Free Loss regularized autoencoder, the unregularized autoencoder, and the Tikhonov regularized autoencoder. These can seen in Figure~\ref{fig:autoencoder}. As we can see the Free Loss version, is the only autoencoder that Gaussianizes the latent code.

\subsection{Real Data}
\label{sec:E2}

We also train autoencoders for real image data. We use the same EfficientNet ViT from before. For the decoder, we use the following SimpleLatentDecoder. This architecture was created by ChatGPT to act a simple decoder.

\textbf{Input.} Latent vector \( z \in \mathbb{R}^{\texttt{embedding\_dim}} \) (e.g., 32).

\textbf{Initial Projection.} A linear layer maps from \texttt{embedding\_dim} to \texttt{base\_ch} \(\times 7 \times 7\), followed by Gaussian Error Linear Unit (GELU) activation. The output is reshaped to \texttt{(batch, base\_ch, 7, 7)}.

\textbf{Mixing at 7x7.} We then perform a 1x1 convolution (pointwise), GroupNorm with 1 group (equivalent to LayerNorm over channels), and GELU activation.

\textbf{Upsampling Blocks.} A series of five UpBlock modules, progressively upsampling the spatial dimensions from $7 \times 7$ to $224 \times 224$ while halving the channels approximately each time: 

\begin{align*}
\texttt{base\_ch} &\to \texttt{base\_ch//2} \\
&\to \texttt{base\_ch//4} \\
&\to \texttt{base\_ch//8} \\
&\to \max(\texttt{base\_ch//16}, 32) \\
&\to \max(\texttt{base\_ch//32}, 32).
\end{align*}

Each UpBlock consists of:
\begin{itemize}
\item Upsampling by a factor of 2 (default: bilinear interpolation).
\item 1x1 projection convolution to output channels.
\item Depthwise 3x3 convolution (groups = channels).
\item GroupNorm with 1 group.
\item GELU activation.
\item Addition of a residual connection from after the projection, plus a Swish-Gated Linear Unit (SwiGLU) 2D module applied to the normalized output.
\end{itemize}

The SwiGLU2D is a minimal MLP over channels using 1x1 convolutions: input projection to twice the expanded channels, split into value and gate, gate passed through Sigmoid Linear Unit (SiLU) and multiplied by value, then output projection back to original channels, with optional dropout.

\textbf{Head.} A final refinement sequence at 224x224: 3x3 convolution (padding=1), GroupNorm with 1 group, GELU, and 1x1 convolution to \texttt{out\_channels} (e.g., 3 for RGB).

We then trained a Free Loss regularized autoencoder for MNIST, CIFAR, CelebA, and Imagenet. For each dataset we used 50,000 training data points, a batch size $b = 128$ and an embedding dimension of $d = 96$. We trained for 50 epochs using Adam with a learning rate of $10^{-3}$. For MNIST and CIFAR we used $\tau = 0.1$ and for CelebA and Imagenet we used $\tau = 0.01$. The deviation from Gaussianity statistics and the MSE can be seen in Table~\ref{tab:auto-real}. Note in all cases, we managed to Gaussianize the code. The equivalent of Figure~\ref{fig:2} for the autoencoder can also be seen in Figure 16 and 17 for MNIST and CIFAR data respectively. 

{
\sisetup{
  table-number-alignment = center,
  round-mode = places,
  round-precision = 4
}
\begin{table*}[t]
\centering
\begin{tabular}{l *{8}{S}}
\toprule
{Dataset} & \multicolumn{2}{c}{Relative MSE} & \multicolumn{2}{c}{Relative Free Loss} & \multicolumn{2}{c}{Relative OT $\Delta_{\texttt{OT}}$} & \multicolumn{2}{c}{KS} \\
\cmidrule(lr){2-3}\cmidrule(lr){4-5}\cmidrule(lr){6-7}\cmidrule(lr){8-9}
{} & {Train} & {Test} & {Train} & {Test} & {Train} & {Test} & {Train} & {Test} \\
\midrule
CIFAR
  & 0.0407 & 0.0730
  & 0.0016868897946551442 & 0.0008239990565925837
  & 0.0025621536187827587 & 0.003994377795606852
  & 0.007957901320060135 & 0.011370909747841385 \\
MNIST
  & 0.0072 & 0.001
  & 0.0034935344010591507 & 0.007954721339046955
  & 0.005482370033860207 & 0.0023901890963315964
  & 0.015024904319962895 & 0.012181615705468107 \\
CelebA
  & 0.0918 & 0.1057
  & 0.00196371553465724 & 0.019124414771795273
  & 0.008947089314460754 & 0.01956791616976261
  & 0.0074573610268383606 & 0.010874123246603029 \\
\bottomrule
\end{tabular}
\caption{Training and test Relative Mean Squared Error $\|\mathcal{D}(\mathcal{E}(X))-X\|_F^2/\|X\|_F^2$ and Gaussianization deviation statistics per dataset.}
\label{tab:auto-real}
\end{table*}
} 

\subsection{More Challenging Real Data}
\label{sec:E3}

We end with a note that some datasets are more challenging to auto-encode with latent Gaussian codes. For example, if we use our model for CelebA, while we can successfully Gaussianize the latent code, we have poor reconstruction. We believe that this is an issue of model capacity or embedding dimension. Increasing both, should resolve this issue, and we leave it for future work.  









\begin{figure*}[t]
  \centering

  \begin{subfigure}[t]{\textwidth}
    \centering
    \begin{subfigure}[t]{0.32\textwidth}
      \centering
      \includegraphics[width=\linewidth]{figs/hist-epoch-tau-1-mml2000-test.png}
      \caption{Histogram of Entries}
      \label{fig:tik1-hist}
    \end{subfigure}\hfill
    \begin{subfigure}[t]{0.32\textwidth}
      \centering
      \includegraphics[width=\linewidth]{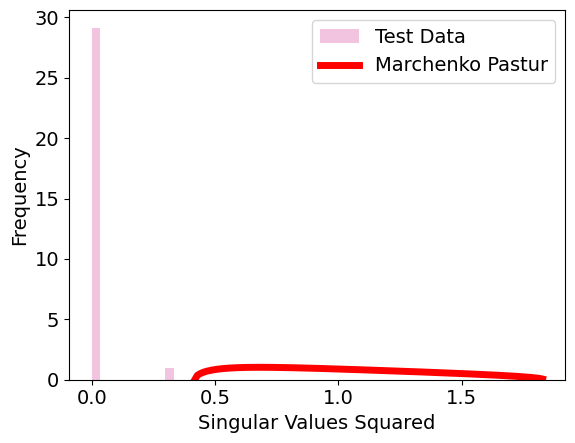}
      \caption{Histogram for Eigenvalue}
      \label{fig:tik1-eigs}
    \end{subfigure}\hfill
    \begin{subfigure}[t]{0.32\textwidth}
      \centering
      \includegraphics[width=\linewidth]{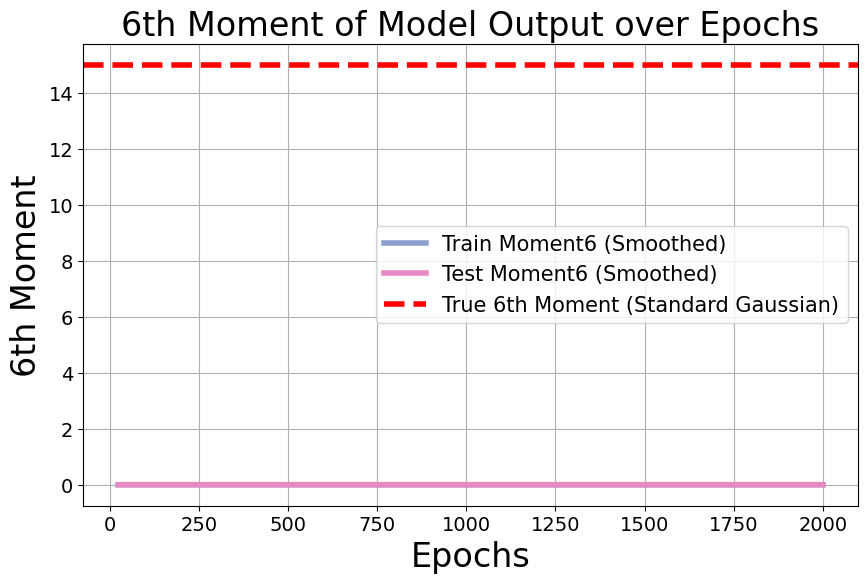}
      \caption{6th Moment}
      \label{fig:tik1-m6}
    \end{subfigure}
    \caption{Gaussianization statistics for the Tikhonov regularized autoencoder with $\tau=1$.}
    \label{fig:tik1-row}
  \end{subfigure}

  \vspace{0.5em}

  \begin{subfigure}[t]{\textwidth}
    \centering
    \begin{subfigure}[t]{0.32\textwidth}
      \centering
      \includegraphics[width=\linewidth]{figs/hist-epoch-tau-0-autoencoder2000-test.png}
      \caption{Histogram of Entries}
      \label{fig:tau0-hist}
    \end{subfigure}\hfill
    \begin{subfigure}[t]{0.32\textwidth}
      \centering
      \includegraphics[width=\linewidth]{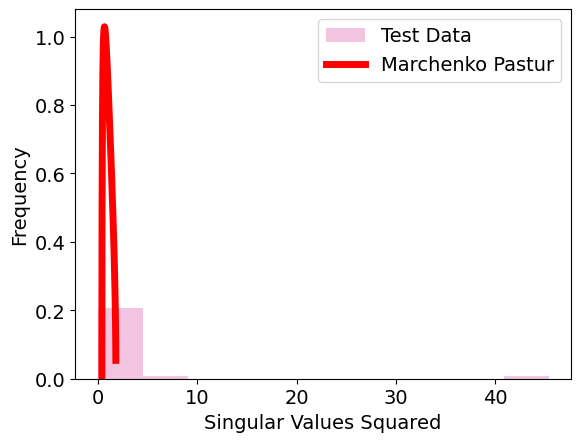}
      \caption{Histogram for Eigenvalue}
      \label{fig:tau0-eigs}
    \end{subfigure}\hfill
    \begin{subfigure}[t]{0.32\textwidth}
      \centering
      \includegraphics[width=\linewidth]{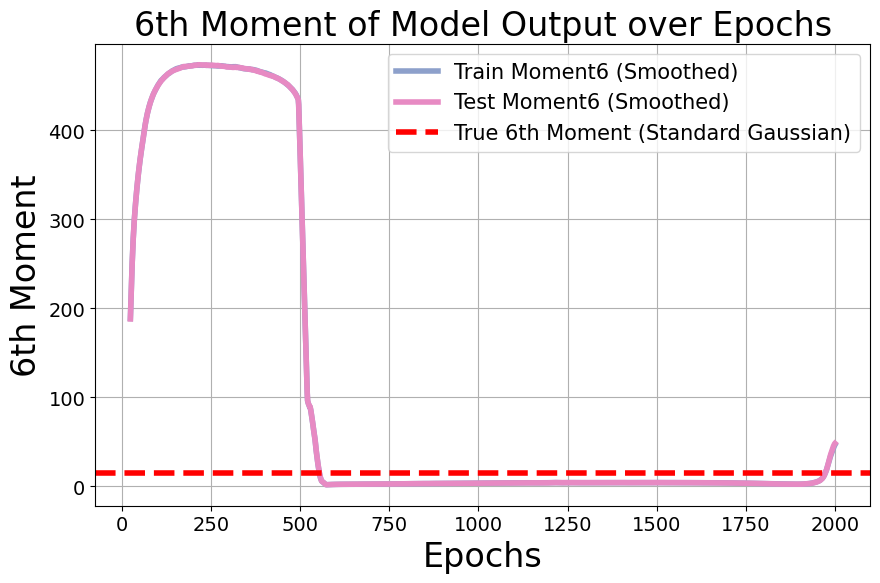}
      \caption{6th Moment}
      \label{fig:tau0-m6}
    \end{subfigure}
    \caption{Gaussianization statistics for the unregularized autoencoder ($\tau=0$).}
    \label{fig:tau0-row}
  \end{subfigure}

  \vspace{0.5em}

  \begin{subfigure}[t]{\textwidth}
    \centering
    \begin{subfigure}[t]{0.32\textwidth}
      \centering
      \includegraphics[width=\linewidth]{figs/hist-epoch-tau-1-autorencoder2000-test.png}
      \caption{Histogram of Entries}
      \label{fig:free1-hist}
    \end{subfigure}\hfill
    \begin{subfigure}[t]{0.32\textwidth}
      \centering
      \includegraphics[width=\linewidth]{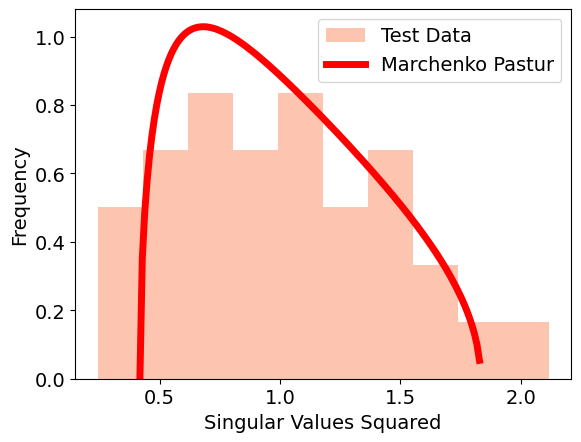}
      \caption{Histogram for Eigenvalue}
      \label{fig:free1-eigs}
    \end{subfigure}\hfill
    \begin{subfigure}[t]{0.32\textwidth}
      \centering
      \includegraphics[width=\linewidth]{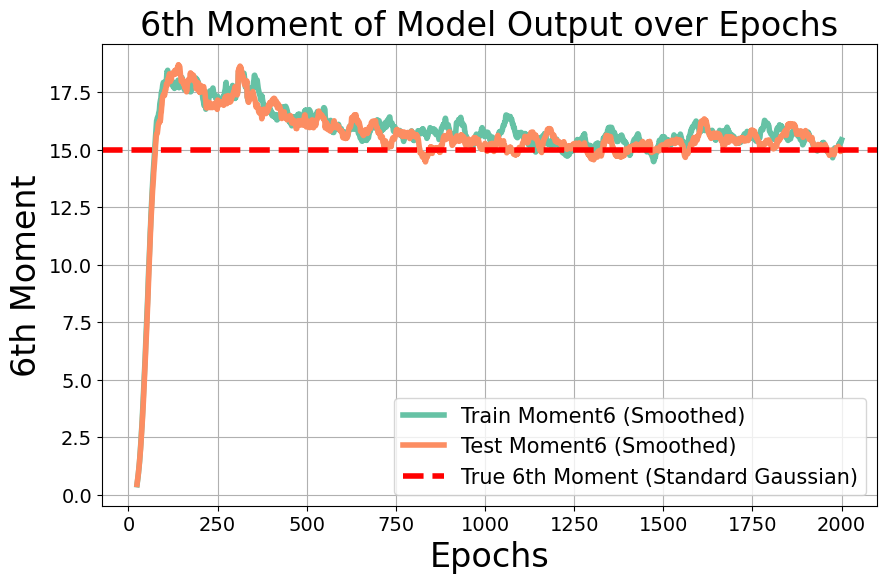}
      \caption{6th Moment}
      \label{fig:free1-m6}
    \end{subfigure}
    \caption{Gaussianization statistics for the Free loss regularized autoencoder with $\tau=1$.}
    \label{fig:free1-row}
  \end{subfigure}

  \caption{Gaussianization statistics comparing the unregularized autoencoder, the Tikhonov-regularized autoencoder, and the Free-loss regularized autoencoder. See Section~\ref{sec:E1} for more details}
  \label{fig:autoencoder}
\end{figure*}

\MakeAEEvolveFig{auto}{MNIST}{MNIST}{10}{$c=32/128$}
\MakeAEEvolveFig{auto}{CIFAR}{CIFAR}{10}{$c=32/128$}

\end{document}